\documentclass{article}

\usepackage[preprint]{styles/corl_2026}

\usepackage{amsmath}
\usepackage{amssymb}
\usepackage{amsthm}
\usepackage{mathtools}
\usepackage{bm}
\usepackage{microtype}

\usepackage{graphicx}
\usepackage{booktabs}
\usepackage{multirow}
\usepackage{makecell}
\usepackage{subcaption}
\usepackage{wrapfig}

\usepackage{algorithm}
\usepackage{algpseudocode}

\usepackage{enumitem}
\usepackage{xcolor}
\usepackage[table]{xcolor}
\usepackage{url}
\usepackage{hyperref}
\hypersetup{
    colorlinks=true,
    linkcolor=blue!60!black,
    citecolor=green!50!black,
    urlcolor=blue!60!black,
}
\usepackage{titletoc}
\usepackage{cleveref}
\crefname{appendix}{Appendix}{Appendices}
\Crefname{appendix}{Appendix}{Appendices}
\usepackage{pgfplots}
\usepgfplotslibrary{polar}
\pgfplotsset{compat=1.18}

\newtheorem{theorem}{Theorem}
\newtheorem{proposition}[theorem]{Proposition}

\newtheorem{corollary}[theorem]{Corollary}
\theoremstyle{definition}
\newtheorem{definition}[theorem]{Definition}
\newtheorem{assumption}[theorem]{Assumption}
\theoremstyle{remark}
\newtheorem{remark}[theorem]{Remark}

\newcommand{\equivla}{\textsc{EquiVLA}}

\newcommand{\equiperceptor}{\textsc{EquiPerceptor}}
\newcommand{\equiactor}{\textsc{EquiActor}}

\newcommand{\SO}{\mathrm{SO}}
\newcommand{\SE}{\mathrm{SE}}

\newcommand{\norm}[1]{\left\lVert #1 \right\rVert}
\newcommand{\set}[1]{\{#1\}}

\usepackage{pifont}

\title{EquiVLA: A General Framework for
Rotationally Equivariant Vision-Language-Action Models}
\author{
  \parbox{\textwidth}{\centering
  \vspace{0.2in}
    Thien-Loc Ha$^{1\ast}$ \quad
    Quang-Tan Nguyen$^{1\ast}$ \quad
    Trong-Bao Ho$^{1\ast}$ \quad
    Long Dinh$^{1,2}$ \\[0.15em]
    Minh Duc Nguyen$^{1,2}$ \quad
    Gia-Binh Nguyen$^{1,2}$ \quad
    Pham Tri Quang$^{1}$ \quad
    Minh N. Vu$^{1,2}$ \\[0.15em]
    Duy M. H. Nguyen$^{3,4,5}$ \quad
    An Thai Le$^{1,2}$ \quad
    Ngo Anh Vien$^{1,2}$ \\[1.2em]
    \footnotesize\normalfont
    $^{1}$VinRobotics \quad
    $^{2}$VinUniversity \quad
    $^{3}$DFKI \quad
    $^{4}$University of Stuttgart \quad
    $^{5}$IMPRS-IS \\[0.4em]
    \footnotesize\normalfont $^{\ast}$Equal Contributors.
  }
}
\begin{document}

\maketitle

\begin{abstract}
Vision-Language-Action (VLA) models have emerged as a powerful paradigm for generalist
robot manipulation, yet they lack geometric inductive biases: policies trained at specific
orientations require substantially more data to generalize across rotational configurations.
We present \equivla{}, the first general framework for end-to-end $\SO(2)$-equivariant VLA
models, applicable to any architecture coupling a frozen vision-language backbone with a
flow-matching Diffusion Transformer action head. \equivla{} introduces \equiperceptor{}, which produces approximately $\SO(2)$-equivariant visual
representations from frozen ViT features; and \equiactor{}, an exactly
$\SO(2)$-equivariant flow-matching Diffusion Transformer action head. Together, they
establish an approximate $\SO(2)$ equivariance chain from camera observations to predicted
action sequences. Instantiated on GR00T~N1.5 and evaluated across four LIBERO suites,
CALVIN ABCD$\to$D, and five real-robot tasks on Mobile ALOHA, \equivla{} achieves $92.6\%$
average success on LIBERO (vs.\ $78.1\%$ baseline), an average sequence length of $4.03$
on CALVIN (vs.\ $3.45$), and improves real-robot success from $54\%$ to $72\%$. Project page: \url{https://equivla.github.io/}
\end{abstract}


\section{Introduction}
\label{sec:intro}
Vision-Language-Action models~\citep{bjorck2025gr00t, 
black2024pi0, kim2024openvla} have emerged as a powerful paradigm for generalist robot manipulation, achieving strong performance across diverse tasks by conditioning flow-matching action heads on rich representations from pretrained vision-language backbones. These models demonstrate impressive generalization through language-conditioned imitation learning, enabling robots to follow natural-language instructions across diverse objects and environments. However, despite their capabilities, current VLA models face a fundamental bottleneck: generalization to novel object orientations remains data-intensive, as policies trained at specific orientations must see substantially more 
data to generalize across rotational configurations encountered in deployment.

The root cause of this bottleneck lies in a geometric structure that current VLA models fail to exploit. Robot manipulation tasks possess an inherent rotational symmetry: rotating the workspace and its objects yields an equivalent task that requires a correspondingly rotated action. This SO(2) structure means that knowledge learned at one orientation should transfer immediately to all others, yet current VLA architectures ignore this property entirely. Their VLM backbones process images without any awareness of rotational structure, encoding visual content at fixed spatial locations that change meaning under rotation, and their DiT action heads apply unconstrained transformations that treat rotationally related observations as independent inputs. The result is a model that must independently learn the same manipulation skill at every orientation it encounters - a structural absence of geometric inductive bias that data augmentation alone can mitigate but not fully resolve, as augmentation introduces no architectural guarantee of equivariance~\citep{wang2025practical}.

Recent work has demonstrated that encoding \(\SO(2)\) equivariance directly into policy architectures yields substantial benefits in both data efficiency and task performance, with equivariant policies consistently outperforming non-equivariant baselines across manipulation benchmarks~\citep{wang2024equidiff, yang2024equibot, wang2022so2rl, tie2025etseed}. However, existing equivariant policy methods are limited to architectures trained from scratch on image observations or point clouds~\citep{wang2024equidiff, yang2024equibot, zhu2025equact}, or apply canonicalization as a model-agnostic wrapper~\citep{deng2025eqbot}; none addresses the full VLM-plus-DiT pipeline that defines modern large-scale VLA models, leaving the most capable robot learning systems without the geometric structure that has proven so effective at a smaller scale.

To bridge the gap, we present \equivla{}, a general framework for end-to-end
$\SO(2)$-equivariant VLA models designed for architectures
coupling a frozen VLM backbone with a flow-matching DiT action head.
\equivla{} comprises two composable modules, \equiperceptor{} and
\equiactor{}, establishing an approximate $\SO(2)$ equivariance chain
from raw camera images to predicted action sequences without modifying
pretrained VLM weights. We instantiate \equivla{} on GR00T~N1.5~\citep{bjorck2025gr00t} and evaluate across LIBERO~\citep{liu2023libero}, CALVIN~\citep{mees2022calvin}, and real-robot tasks on Mobile ALOHA~\citep{fu2024mobile}. \equivla{} achieves $92.6\%$ average success on LIBERO under relative control (vs.\ $78.1\%$ for the baseline), an average sequence length of $4.03$ on CALVIN ABCD$\to$D (vs.\ $3.45$ for the baseline, approaching history-based methods without temporal context), and improves average real-robot success from $54\%$ to $72\%$ across five tasks. We make several contributions in this work:
\begin{enumerate}[leftmargin=*,topsep=2pt,itemsep=1pt]
    \item \textbf{Token-level Frame Averaging} (\equiperceptor{}): an extension of Frame Averaging (FA)~\citep{puny2021frame} from globally pooled vectors to spatially indexed ViT patch token sequences, yielding approximately $\SO(2)$-equivariant visual representations from any frozen VLM.
    \item \textbf{Equivariant Flow-matching DiT} (\equiactor{}): the first $\SO(2)$-equivariant Diffusion Transformer action head for robot policy learning, with equivariant attention, state encoding, and action decoding via steerable layers in regular feature space.
    \item \textbf{End-to-end Equivariant VLA} (\equivla{}): a framework combining \equiperceptor{} and \equiactor{} to establish 
    an approximate $\SO(2)$ equivariance chain from camera observations 
    to action sequences, demonstrated on GR00T~N1.5 across simulation 
    and real-robot benchmarks.
\end{enumerate}

\section{Related Work}
\label{sec:related}
\textbf{Vision-Language-Action Models.}
VLA models unify visual perception, language understanding, and action generation for generalist robot manipulation~\citep{brohan2022rt, zitkovich2023rt, team2024octo}. Early autoregressive VLAs~\citep{kim2024openvla, brohan2022rt, zitkovich2023rt} tokenize actions with LLM backbones, while more recent diffusion and flow-matching VLAs~\citep{bjorck2025gr00t, black2024pi0, intelligence2025pi_, liu2024rdt,
li2024cogact, shukor2025smolvla, kim2025fine} synthesize continuous action sequences via iterative denoising~\citep{chi2023diffusion} or flow matching, yielding richer action distributions and smoother trajectories.
Despite their strong performance, no existing VLA encodes rotational symmetry
in the architecture, instead relying on data augmentation or increased
demonstration diversity to generalize across object orientations, a costly
substitute for the systematic guarantees that equivariance provides. Concurrent work on embodiment equivariance~\citep{intelligence2025pi_} addresses cross-embodiment transfer but does not encode spatial symmetries within a single embodiment's observation-action pipeline. Our work addresses this gap by imposing $\SO(2)$ equivariance on both the VLM visual backbone and the flow-matching action head, enabling generalization to unseen orientations by construction rather than by memorization.

\textbf{Equivariant Robot Policies.}
Equivariant policy learning exploits geometric symmetries to improve sample efficiency and generalization~\citep{wang2025practical, wang2024equidiff, wang2022so2rl, yang2024equibot, huang2022equivariant, jia2022seil}.
Prior work establishes $\SO(2)$ and $\mathrm{SIM}(3)$ equivariance across a range of policy architectures, from pick-and-place~\citep{huang2022equivariant, jia2022seil} to 6-DoF closed-loop diffusion policies~\citep{wang2024equidiff}, trajectory-level $\SE(3)$-equivariant diffusion~\citep{tie2025etseed}, and point-cloud-based policies~\citep{yang2024equibot}, which consistently reduce data requirements and improve generalization to novel object poses.
EquAct~\citep{zhu2025equact} introduces an $\SE(3)$-equivariant point transformer for multi-task manipulation with language conditioning via invariant FiLM layers, achieving strong performance on RLBench~\citep{james2020rlbench}, but operates on point clouds trained from scratch rather than leveraging pretrained VLM representations.
Eq.Bot~\citep{deng2025eqbot} takes a model-agnostic canonicalization approach, transforming observations into a canonical orientation before feeding them to an unmodified policy; while convenient, canonicalization provides only approximate equivariance that depends on the quality of frame estimation and does not propagate geometric structure through the action head.
All of the above methods either train compact equivariant architectures from scratch or wrap existing policies with canonicalization; none integrates equivariance into the full pipeline of a large-scale VLA model combining a pretrained VLM backbone with a flow-matching DiT action head.

\textbf{Symmetrizing Pretrained Models.}
To impose symmetry on pretrained models without retraining, Frame Averaging (FA)~\citep{puny2021frame} averages predictions over a group orbit of transformed inputs; \citet{wang2025practical} apply FA to ResNet's global image embedding in diffusion policy, demonstrating that even simple symmetry incorporation significantly improves sample efficiency. However, both works apply FA only to globally pooled single vectors and cannot be directly extended to ViT~\citep{dosovitskiy2020image} spatial token sequences, where rotating the input displaces each patch token to a different grid position and naively averaging mixes spatially misaligned tokens, destroying the positional structure that manipulation policies require.
EquiLLM~\citep{li2025equillm} injects equivariance into a frozen LLM by routing only invariant features through the backbone while handling directional information via an equivariant GNN encoder, but operates on 3D molecular systems rather than visual robot manipulation.
Our work brings $\SO(2)$ equivariance to the full pipeline of a VLA model by adapting the invariant-injection principle of EquiLLM~\citep{li2025equillm} to the VLA setting, extending FA to spatially indexed ViT patch token sequences and introducing an equivariant flow-matching DiT action head, to our knowledge, the \emph{first} framework to impose end-to-end $\SO(2)$ equivariance on a large-scale VLA model without modifying pretrained VLM weights.

\section{Method}
\label{sec:method}
\equivla{} imposes $\SO(2)$ equivariance on VLA architectures that 
couple a pretrained VLM backbone with a flow-matching action head 
through two composable modules, with no modification to pretrained 
VLM weights. \Cref{fig:equivla} illustrates the architecture; 
\equiperceptor{} is described in~\Cref{sec:method:equiperceptor} 
and \equiactor{} is in~\Cref{sec:method:equiactor}.

\begin{figure*}[htbp]
    \centering
    \includegraphics[width=\linewidth]{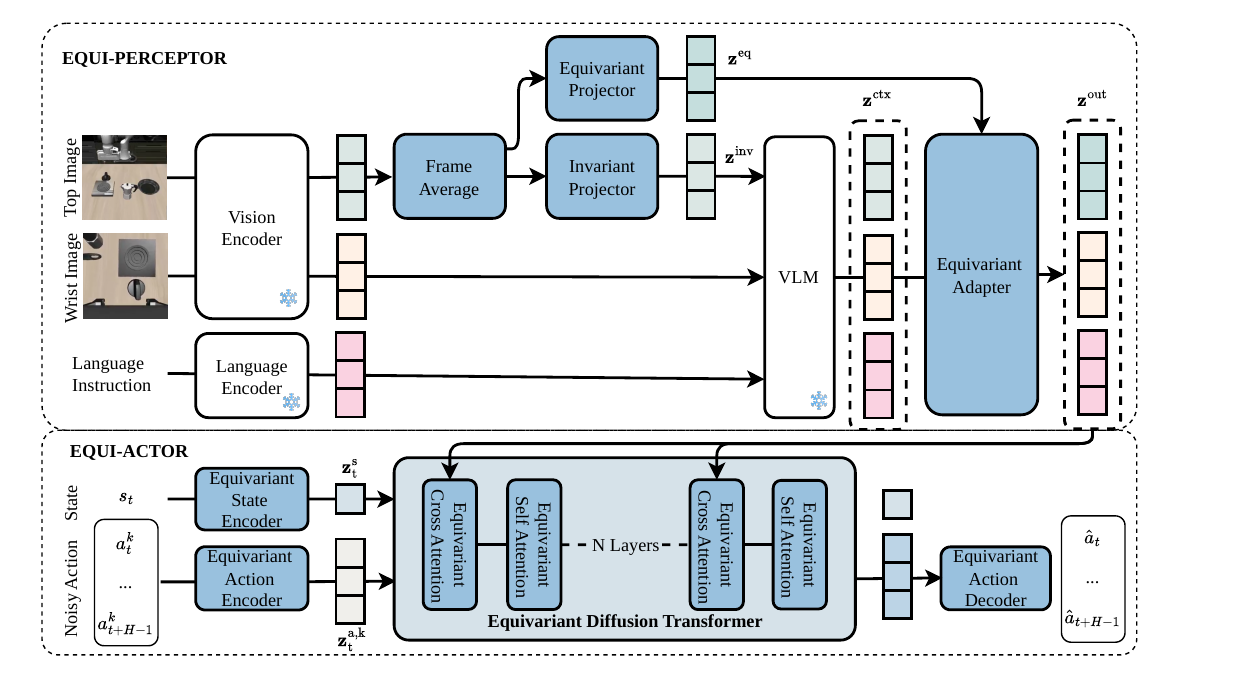}
    \caption{
    \textbf{EquiVLA architecture.}
    (\textit{Top}) \equiperceptor{} symmetrizes the frozen VLM's visual pipeline via
Token-level Frame Averaging, producing equivariant tokens $\mathbf{z}^{\mathrm{eq}}$ and
invariant tokens $\mathbf{z}^{\mathrm{inv}}$. The latter are fed into the frozen VLM
alongside the wrist image and language instruction to produce context tokens
$\mathbf{z}^{\mathrm{ctx}}$, fused with $\mathbf{z}^{\mathrm{eq}}$ via the
Equivariant Adapter to form the visual representation passed to \equiactor{}.
(\textit{Bottom}) \equiactor{} replaces the standard DiT action head with an
$\SO(2)$-equivariant Diffusion Transformer, refining state and noisy action tokens through
$N$ equivariant cross- and self-attention layers conditioned on visual context, then
decoding actions $\hat{\mathbf{a}}_t$ via flow matching.
    }
    \label{fig:equivla}
\end{figure*}

\subsection{Preliminaries}
\label{sec:method:prelim}

\textbf{Equivariance.}
A function $f$ is equivariant if it commutes with the
transformations of a symmetry group $G$: for all $g \in G$,
$f(\rho_{\mathrm{in}}(g)\mathbf{x}) = \rho_{\mathrm{out}}(g)f(\mathbf{x})$,
where $\rho: G \to \mathrm{GL}(n)$ is a \emph{group representation}
mapping each group element to an invertible matrix acting on the feature
space. Invariance is the special case $\rho_{\mathrm{out}}(g) = I$.
We sometimes write this compactly as $f(g \cdot \mathbf{x}) = g \cdot
f(\mathbf{x})$.

\textbf{Symmetry Group.} We focus on $\SO(2)$ and its finite cyclic 
subgroup $G = C_u \leq \SO(2)$ containing $u$ discrete rotations 
$\{1, r, r^2, \dots, r^{u-1}\}$ where $r = R_{360^\circ/u}$. 
In all experiments, we use $C_8$; this choice is ablated 
in~\Cref{sec:ablations}.

\textbf{Representations.}
Three representations of $C_u$ are used throughout this paper.
The \emph{trivial representation} $\rho_0$ acts on an invariant 
scalar $x \in \mathbb{R}$ by $\rho_0(g)x = x$.
The \emph{irreducible representation} $\rho_\omega$ acts on a vector 
$\mathbf{v} \in \mathbb{R}^2$ by a $2{\times}2$ rotation matrix at 
frequency $\omega$:
$\rho_\omega(g)\mathbf{v} = \begin{pmatrix} \cos\omega g & -\sin\omega g \\ \sin\omega g & \cos\omega g \end{pmatrix}\mathbf{v}$;
in particular, $\rho_1$ is the standard 2D rotation and $\rho_2$ 
rotates at twice the frequency.
The \emph{regular representation} $\rho_{\mathrm{reg}}$ acts on
$\mathbf{x} \in \mathbb{R}^u$ by cyclically permuting its coordinates:
for $g = r^m \in C_u$,
$\rho_{\mathrm{reg}}(g)\mathbf{x} =
(x_{u-m+1}, \dots, x_u, x_1, \dots, x_{u-m})$. Representations combine via direct sum $\rho_a \oplus \rho_b$,
where each component transforms independently.
All equivariant layers are implemented using steerable 
CNNs~\citep{weiler2019general, weiler2023equivariant, cesa2022escnn}.

\subsection{\equiperceptor{}: Equivariant Visual Token Fusion}
\label{sec:method:equiperceptor}

\equiperceptor{} builds on the core insight of EquiLLM~\citep{li2025equillm}: a frozen 
language model can act as an invariant feature processor, provided only invariant features 
are fed into it while equivariant information flows through a separate stream. However, 
EquiLLM was designed for 3D molecular systems using a geometric GNN encoder, and does not 
transfer directly to the vision-language pipeline of robot manipulation. We adapt this 
principle through two novel components.



\textbf{Token-level Frame Averaging.} A ViT outputs a token sequence $\{\mathbf{z}_i\}_{i=1}^{N}$ in which each token is tied to a specific spatial patch. When the input image is rotated by a group element $h \in G$, each patch moves to a different grid position. Therefore, naively averaging token sequences across different group elements averages tokens that no longer correspond to the same spatial location. This produces an invariant representation, but it also removes the spatial localization needed for manipulation.
The inverse group element $h^{-1}$ must therefore be applied to ``undo'' each
transformation before averaging, acting simultaneously at two levels: 1) a spatial
permutation $\tau(h^{-1})$ that maps each displaced token back to its canonical
patch position, and 2) a feature-space transformation $\rho_{\mathrm{reg}}(h^{-1})$
in the regular representation of $G$.
Concretely, $\tau(h^{-1})$ is computed by applying $h^{-1}$ to each patch center
coordinate and assigning the result to the nearest canonical grid index.
This nearest-neighbor assignment introduces a discretization error that is zero when $h$ maps every patch center exactly onto another grid position (i.e., for $C_4$ rotations on a square grid) and bounded by sub-patch displacement otherwise (we quantify this error empirically in~\Cref{sec:ablations}).
Let $f_{\theta}: \mathbb{R}^{H \times W \times 3} \to \mathbb{R}^{N \times D}$
denote the frozen ViT. We define the equivariant and invariant streams as:
\begin{equation}
  \mathbf{z}^{\mathrm{eq}}(\mathbf{x})
  = \frac{1}{|G|}\sum_{h \in G}
    \bigl[\tau(h^{-1}) \otimes \rho_{\mathrm{reg}}(h^{-1})\bigr]
    \cdot f_{\theta}(h \cdot \mathbf{x})\,,
  \label{eq:fa_equi}
\end{equation}
\begin{equation}
  \mathbf{z}^{\mathrm{inv}}(\mathbf{x})
  = \frac{1}{|G|}\sum_{h \in G} f_{\theta}(h \cdot \mathbf{x})\,,
  \label{eq:fa_inv}
\end{equation}
where the tensor product in \Cref{eq:fa_equi} acts jointly on token positions (via $\tau$) and feature channels (via $\rho_{\mathrm{reg}}$).
$\mathbf{z}^{\mathrm{eq}}$ satisfies
$\mathbf{z}^{\mathrm{eq}}(g \cdot \mathbf{x}) =
[\tau(g) \otimes \rho_{\mathrm{reg}}(g)] \cdot \mathbf{z}^{\mathrm{eq}}(\mathbf{x})$
for all $g \in G$ (proof in~\Cref{app:equivariance_proof}), while $\mathbf{z}^{\mathrm{inv}}$
is $G$-invariant by the standard averaging argument~\citep{puny2021frame}.

\textbf{Equivariant and Invariant Projectors} project the two FA token maps into the VLM token dimension via separate lightweight layers: a standard MLP $\psi^{\mathrm{inv}}$
yielding $\tilde{\mathbf{z}}^{\mathrm{inv}} = \psi^{\mathrm{inv}}(\mathbf{z}^{\mathrm{inv}})$, and a $G$-equivariant linear layer $\psi^{\mathrm{eq}}$ yielding $\tilde{\mathbf{z}}^{\mathrm{eq}} = \psi^{\mathrm{eq}}(\mathbf{z}^{\mathrm{eq}})$. $\tilde{\mathbf{z}}^{\mathrm{inv}}$ is concatenated with the 
wrist-camera and language tokens, then forwarded through the frozen VLM to produce language-grounded context tokens $\mathbf{z}^{\mathrm{ctx}}$. 
Both $\tilde{\mathbf{z}}^{\mathrm{eq}}$ and $\mathbf{z}^{\mathrm{ctx}}$ 
are then passed to the Equivariant Adapter, which we explain next.

\textbf{Equivariant Adapter} fuses $\tilde{\mathbf{z}}^{\mathrm{eq}}$ with
$\mathbf{z}^{\mathrm{ctx}}$ via a learned invariant gate, ensuring all gate inputs are restricted to invariant quantities to preserve equivariance. Three invariant summaries are extracted: mean-pooled language $\mathbf{s}^{\mathrm{lang}}$ and vision $\mathbf{s}^{\mathrm{vis}}$ from $\mathbf{z}^{\mathrm{ctx}}$, and group-pooled equivariant tokens $\bar{\mathbf{z}}^{\mathrm{eq}} = \mathrm{mean}_{G}(\tilde{\mathbf{z}}^{\mathrm{eq}})$.
The gate $\boldsymbol{\alpha} = \sigma(\mathbf{W}_{\mathrm{gate}}[ \mathbf{s}^{\mathrm{lang}}; \mathbf{s}^{\mathrm{vis}}; \bar{\mathbf{z}}^{\mathrm{eq}}])$, tiled to regular representation $\boldsymbol{\alpha}^{\mathrm{reg}} = \boldsymbol{\alpha} \otimes \mathbf{1}_{|G|}$ (Kronecker product replicating $\boldsymbol{\alpha}$ 
across $|G|$ group channels), blends a semantic branch with an equivariant branch:
\begin{equation}
    \mathbf{z}^{\mathrm{out}}_{\mathrm{eq}}
    = \boldsymbol{\alpha}^{\mathrm{reg}}
      \odot \mathbf{W}_s(\mathbf{s}^{\mathrm{inv}} \otimes \mathbf{1}_{|G|})
    + (\mathbf{1} - \boldsymbol{\alpha}^{\mathrm{reg}})
      \odot \mathbf{W}_g(\tilde{\mathbf{z}}^{\mathrm{eq}}) \,,
    \label{eq:adapter_out}
\end{equation}
where $\mathbf{s}^{\mathrm{inv}} = (\mathbf{s}^{\mathrm{lang}} +
\mathbf{s}^{\mathrm{vis}})/2$.
Since $\boldsymbol{\alpha}^{\mathrm{reg}}$ is invariant and $\mathbf{W}_g$ is
$G$-equivariant, $\mathbf{z}^{\mathrm{out}}_{\mathrm{eq}}$ is equivariant (Proof is in~\Cref{prop:adapter_equivariance}).
Context tokens $\mathbf{z}^{\mathrm{ctx}}$ are similarly updated by blending with an equivariant summary $\mathbf{s}^{\mathrm{eq}} =
\mathrm{mean}_{N,G}(\mathbf{z}^{\mathrm{out}}_{\mathrm{eq}})$ via a per-token
invariant gate (see Appendix~\Cref{app:adapter_inv_stream}). 
The final output $\mathbf{z}^{\mathrm{out}} = [\mathbf{z}^{\mathrm{out}}_{\mathrm{eq}}; \mathbf{z}^{\mathrm{out}}_{\mathrm{inv}}; \mathbf{z}^{\mathrm{out}}_{\mathrm{lang}}]$ is passed as cross-attention context to \equiactor{}.

\subsection{\equiactor{}: \(\SO(2)\)-Equivariant Flow-Matching Action Head}
\label{sec:method:equiactor}

\equiactor{} takes context tokens $\mathbf{z}^{\mathrm{out}}$ from \equiperceptor{}, proprioceptive state embedding $\mathbf{z}^s_t$, and noisy action tokens $\mathbf{z}^{a,k}_t$, and outputs action chunk $\hat{\mathbf{a}}_t$ via flow matching. It replaces the standard DiT action head with a fully $\SO(2)$-equivariant counterpart: all linear projections, attention $\mathbf{Q}/\mathbf{K}/\mathbf{V}$ matrices, state/action encoders, and the action decoder are replaced by $G$-steerable layers in the regular feature space. Each token's feature vector transforms under $\rho_{\mathrm{reg}}(g)$, while token positions are unchanged. Temporal ordering of state and action tokens is encoded via learned position embeddings in the trivial representation, projected equivariantly to regular feature space; since temporal order is invariant to scene rotation, preserving $\SO(2)$ equivariance.

\textbf{Action Representation.}
We adopt the $\SO(2)$-equivariant action representation 
from~\citet{wang2024equidiff}, where end-effector position and 
orientation transform as $\SO(2)$ vectors under scene rotation, 
gripper width is rotation-invariant, and the group action on 
$\mathbf{a}_t$ decomposes into irreducible representations as:
\begin{equation}
    g \cdot \mathbf{a}_t =
    \begin{cases}
        \left(\rho_1^3 \oplus (\rho_1 \oplus \rho_0) \oplus \rho_0\right)(g)\,
        \mathbf{a}_t
        & \text{absolute control} \\[4pt]
        P^{-1}
        \left(\rho_0^6 \oplus \rho_1^4 \oplus \rho_2\right)(g)\,
        P\, \mathbf{a}_t
        & \text{relative control}
    \end{cases}
    \label{eq:action_repr}
\end{equation}
where $P$ reorders action components into irreducible blocks, 
For absolute control, $\rho_1^3$ encodes the 6D end-effector 
rotation as three 2D vector pairs, $\rho_1 \oplus \rho_0$ encodes 
$xy$ translation and $z$ height, and $\rho_0$ encodes gripper width.
For relative control, $\rho_0^6$ encodes invariant scalar offsets 
and $\rho_2$ captures the frequency-2 component from quadratic 
terms in the relative rotation decomposition.

\textbf{Equivariant Attention.}
Attention scores are computed via the geometric inner product $\langle \mathbf{q}, \mathbf{k} \rangle = \sum_g \mathbf{q}[g] \cdot \mathbf{k}[g]$, which is $G$-invariant since $\rho_{\mathrm{reg}}(g)$ is orthogonal. Invariant scores with equivariant $\mathbf{V}$ yield equivariant attention output.

\textbf{Equivariant Self-attention and Cross-attention.} $\mathbf{Q}$, $\mathbf{K}$, $\mathbf{V}$ are projected via regular-to-regular layers and equivariance follows directly from the geometric inner product. For cross-attention, in the full \equivla{} setting, $\mathbf{z}^{\mathrm{out}}$ is equivariant, and all projections are regular-to-regular. In the base VLA~+~\equiactor{} ablation (no \equiperceptor{}), context tokens are invariant VLM outputs; $\mathbf{Q}$ and $\mathbf{K}$ are projected to the trivial representation to keep scores invariant, while $\mathbf{V}$ remains in regular representation. Equivariance of the action token stream is preserved in both cases.


\subsection{\equivla{}: End-to-end Equivariance}

\equivla{} combines \equiperceptor{} and \equiactor{} as two
composable modules. \equiactor{} is trained from scratch, as steerable layers are structurally
incompatible with unconstrained pretrained weights.
This means the action head does not benefit from the usual large-scale pretraining. However, the visual representations from the frozen VLM are preserved intact, and the equivariant inductive bias compensates for the loss of pretrained action-head weights, as demonstrated empirically in \Cref{sec:experiments}.
Together, \equiperceptor{} and \equiactor{} establish an approximate $\SO(2)$ equivariance: for any planar rotation
$g \in C_u$, rotating the camera observations $\mathbf{o}_t$ and proprioceptive state $\mathbf{s}_t$ produces
a correspondingly rotated action sequence,
%

\begin{equation}
    \hat{\mathbf{a}}_t(g \cdot \mathbf{o}_t,\, \rho_s(g)\mathbf{s}_t)
    \approx \rho_a(g) \cdot \hat{\mathbf{a}}_t(\mathbf{o}_t,\, \mathbf{s}_t)
    \quad \forall\, g \in C_u \,,
    \label{eq:equivla_guarantee}
\end{equation}

where $\rho_a(g)$ is defined in \Cref{eq:action_repr}, and 
$\rho_s(g)$ follows the absolute control representation 
therein, as the proprioceptive state is expressed in the world frame. The approximation error stemming from the
approximate $\SO(2)$ equivariance of the frozen ViT in \equiperceptor{} is bounded formally
in~\Cref{thm:approx_equivariance} and~\Cref{prop:exact_cases} shows when the error vanishes. \equiactor{} alone satisfies this relation exactly when
paired with invariant VLM context tokens (\Cref{cor:e2e} with $\Delta = 0$).

\section{Experiments}
\label{sec:experiments}

We experiment on the LIBERO and CALVIN simulation benchmarks and real-robot tasks
with Mobile ALOHA~\citep{fu2024mobile}, comparing three models. \textbf{GR00T~N1.5}~\citep{bjorck2025gr00t} serves
as the non-equivariant baseline. \textbf{GR00T~N1.5 + \equiactor{}} isolates the contribution
of action-head equivariance by replacing only the DiT head. \textbf{\equivla{}} (ours) further
incorporates \equiperceptor{} on top of GR00T~N1.5.
All models share the same pretrained weights and are trained with 
identical hyperparameters (see~\Cref{app:training_details}).

\subsection{LIBERO Benchmark}
\label{sec:experiments:libero}

\textbf{Setup.} We evaluate on four LIBERO suites~\citep{liu2023libero}: LIBERO-10 (long-horizon sequential tasks), LIBERO-Goal (language-conditioned goal specification), LIBERO-Object (object-centric manipulation), and LIBERO-Spatial (spatially precise placement). Each suite contains 10 tasks with approximately 40 demonstrations per task. A separate checkpoint is trained for each suite. We evaluate under both relative and absolute end-effector control, replanning at every timestep, and report success rate averaged over 50 rollouts per task (500 per suite) and two seeds. We additionally include published results for \(\pi_0\)~\citep{black2024pi0}, OpenVLA~\citep{kim2024openvla}, and SmolVLA~\citep{shukor2025smolvla} as reference baselines; note that these use different pretraining data and hyperparameters, so comparisons are indicative rather than controlled.

\begin{wraptable}[14]{R}{0.56\textwidth}
  \vspace{-13pt}
  \centering
  \caption{%
    Success rates (\%) on LIBERO benchmarks. $\dagger$ denotes published results; all others are our runs.
  }
  \label{tab:libero-results}
  \scriptsize
  \setlength{\tabcolsep}{4pt}
  \renewcommand{\arraystretch}{1.1}
  \begin{tabular}{lc ccccc}
    \toprule
    Method & Ctrl & 10 & Goal & Obj & Spat & Avg.\ $\uparrow$ \\
    \midrule
    \(\pi_0\)~\citep{black2024pi0}$^\dagger$         & \multirow{6}{*}{Rel.} & 73.0 & 93.0 & 86.0 & 90.0 & 86.0 \\
    OpenVLA~\citep{kim2024openvla}$^\dagger$          &                       & 55.0 & 79.2 & 88.4 & 84.7 & 76.8 \\
    SmolVLA~\citep{shukor2025smolvla}       &                       & 61.0 & 61.4 & 66.0 & 74.0 & 65.6 \\
    GR00T N1.5~\citep{bjorck2025gr00t}               &                       & 72.0 & 75.0 & 83.4 & 82.0 & 78.1 \\
    \cmidrule(lr){1-1}
    GR00T N1.5 + \equiactor{}                             &  & 82.6 & 88.0 & 95.2 & \textbf{98.2} & 91.0 \\
    \rowcolor{blue!8} \textbf{\equivla{}} (ours)                          &                       & \textbf{87.6} & \textbf{89.4} & \textbf{98.0} & 95.4 & \textbf{92.6} \\
    \midrule
    GR00T N1.5~\citep{bjorck2025gr00t}               & \multirow{3}{*}{Abs.} & 52.0 & 55.2 & 74.6 & 68.6 & 62.6 \\
    GR00T N1.5 + \equiactor{}                             &                       & 63.0 & 70.0 & 79.4 & \textbf{82.0} & 73.6 \\
    \rowcolor{blue!8} \textbf{\equivla{}} (ours)                          &                       & \textbf{73.6} & \textbf{70.4} & \textbf{83.0} & 77.6 & \textbf{76.1} \\
    \bottomrule
  \end{tabular}
\end{wraptable}

\textbf{Results.}
\equivla{} achieves the highest average success rate under both control modes: $92.6\%$ under relative control, surpassing SmolVLA ($65.6\%$), $\pi_0$ ($86.0\%$), and GR00T~N1.5 ($78.1\%$); and $76.1\%$ under absolute control, (+$13.5$\,pp over the GR00T~N1.5 absolute). The gains are progressive: \equiactor{} alone accounts for most of the improvement ($91.0\%$ relative, $73.6\%$ absolute), with \equiperceptor{} contributing a further $1.6$ and $2.5$\,pp, confirming that equivariant visual features complement action-head equivariance.

\subsection{CALVIN Benchmark}
\label{sec:experiments:calvin}
\textbf{Setup.}
We evaluate on the CALVIN ABCD$\to$D benchmark~\citep{mees2022calvin}, which requires
completing chains of five language-conditioned manipulation tasks in a tabletop environment.
Models are trained on environments A, B, C, D, and evaluated zero-shot on held-out 
environment D over 1000 instruction chains, using single-step observations
(image + proprioception). We do not explore multi-frame variants, as visual token sequences
scale proportionally with history length, making it prohibitively expensive at training time.
However, we also report HULC~\citep{mees2022matters} and MoDE~\citep{reuss2024mode} as multi-frame baselines, which are not directly comparable, but help
contextualize benchmark difficulty.

\textbf{Results.}
Without temporal history, GR00T~N1.5 already surpasses HULC ($3.45$ vs.\ $3.07$), 
reflecting the strength of large-scale VLA pretraining. \equiactor{} raises the 
average sequence length to $3.89$ ($+0.44$), and \equivla{} reaches $4.03$ ($+0.58$), 
approaching MoDE~\citep{reuss2024mode} despite using single-frame observations. 
Gains are progressive across all five task positions and largest at the chain's 
tail: Task~5 improves from $48.5\%$ to $64.3\%$ ($+15.8\,\text{pp}$), suggesting that 
equivariance mitigates error accumulation in long-horizon execution.

\begin{wraptable}[10]{R}{0.6\textwidth}
  \centering
  \vspace{-13pt}
  \caption{%
    Average lengths (out of 5.0)
    and per-position success rates (\%) over 1000 chains on CALVIN ABCD$\to$D.
    $\dagger$: published results; $\star$: multi-frame. Tx denotes Task x.
  }
  \label{tab:calvin-results}
  \scriptsize
  \setlength{\tabcolsep}{5pt}
  \renewcommand{\arraystretch}{1.1}
\begin{tabular}{lc ccccc}
\toprule
    Method & T1 & T2 & T3 & T4 & T5 & Avg.\ $\uparrow$ \\
    \midrule
    HULC~\citep{mees2022matters}$^{\dagger\star}$
      & 88.9 & 73.3 & 58.7 & 47.5 & 38.3 & 3.07 \\
    MoDE~\citep{reuss2024mode}$^{\dagger\star}$
      & 97.1 & 92.5 & 87.9 & 83.5 & 77.9 & 4.39 \\
    \cmidrule(lr){1-1}
    GR00T N1.5~\citep{bjorck2025gr00t}
      & 89.0 & 79.2 & 68.7 & 59.4 & 48.5 & 3.45 \\
    GR00T N1.5 + \equiactor{}
      & 93.7 & 85.8 & 77.8 & 70.1 & 61.9 & 3.89 \\
    \rowcolor{blue!8} \textbf{\equivla{}} (ours)
      & \textbf{95.0} & \textbf{88.5} & \textbf{81.1}
      & \textbf{73.8} & \textbf{64.3} & \textbf{4.03} \\
      \bottomrule
\end{tabular}
\end{wraptable}



\subsection{Real Robot Experiments}
\label{sec:experiments:real}

\textbf{Setup.}
We validate \equivla{} on a mobile ALOHA
platform~\citep{fu2024mobile} across five tabletop manipulation tasks, each trained with 150 teleoperated demonstrations. Four tasks use only the right arm, while Shorts Folding requires bimanual control. We compare GR00T~N1.5 and \equivla{} under identical conditions, reporting average success rate over 20 trials per task.

\begin{figure*}[!ht]
  \centering
  \begin{subfigure}[b]{0.48\textwidth}
    \centering
    \includegraphics[width=0.32\linewidth,keepaspectratio]{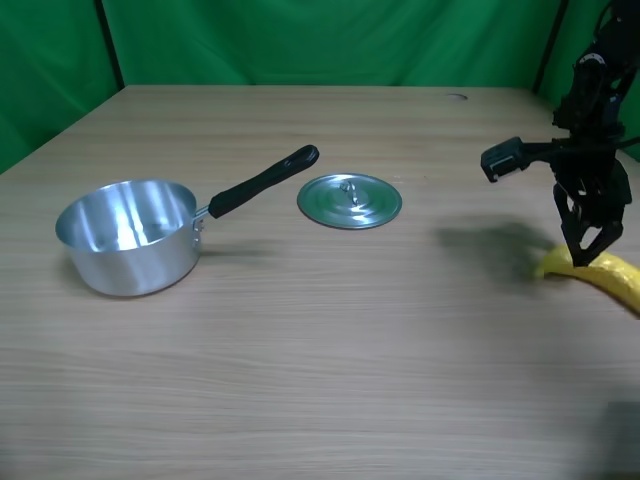}\hfill%
    \includegraphics[width=0.32\linewidth,keepaspectratio]{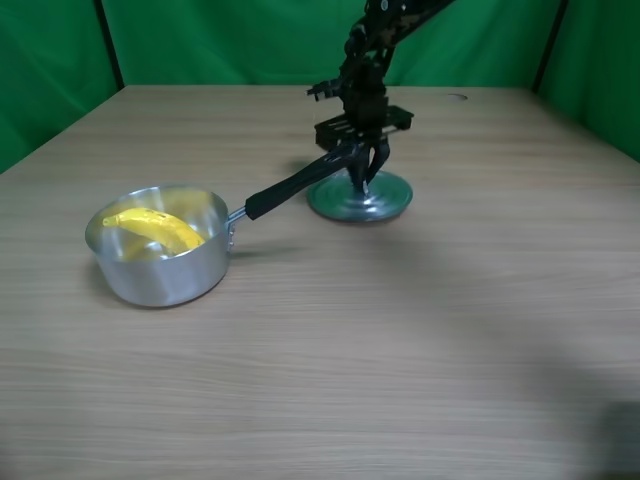}\hfill%
    \includegraphics[width=0.32\linewidth,keepaspectratio]{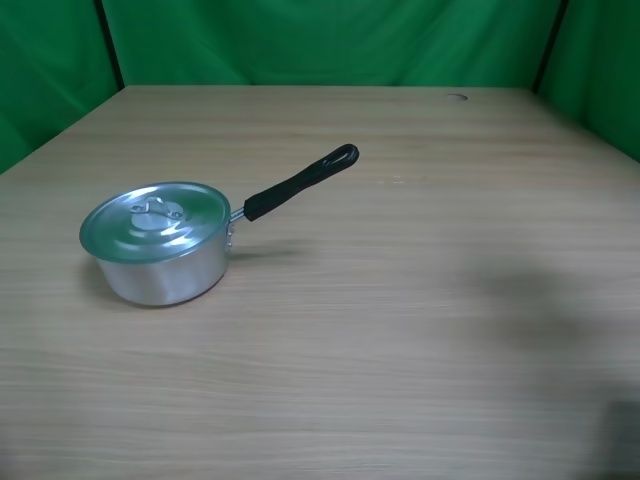}
    \caption{Banana in Pot}
    \label{fig:task-a}
  \end{subfigure}
  \hfill
  \begin{subfigure}[b]{0.48\textwidth}
    \centering
    \includegraphics[width=0.32\linewidth,keepaspectratio]{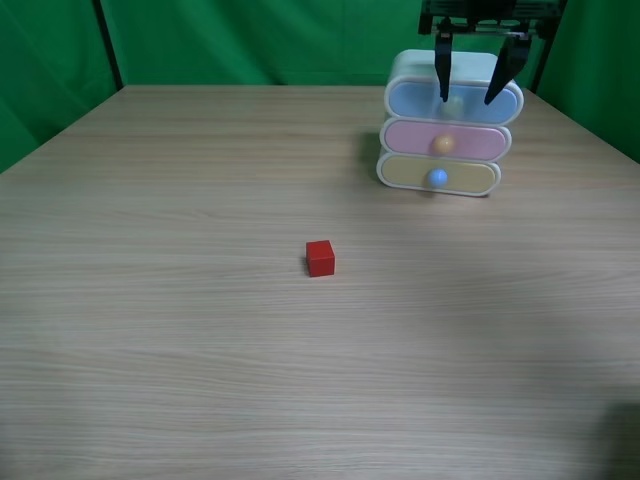}\hfill%
    \includegraphics[width=0.32\linewidth,keepaspectratio]{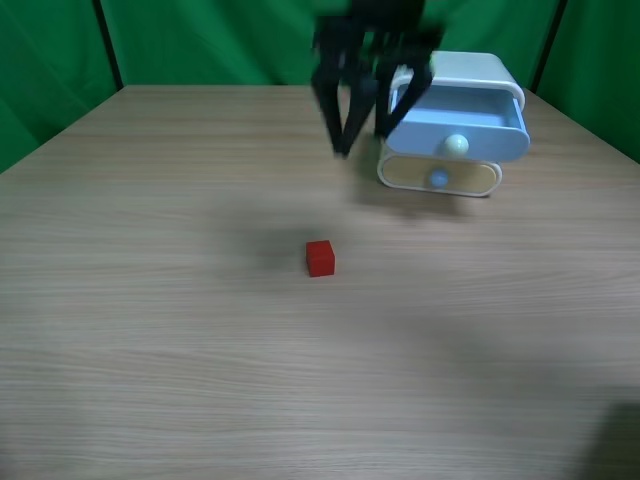}\hfill%
    \includegraphics[width=0.32\linewidth,keepaspectratio]{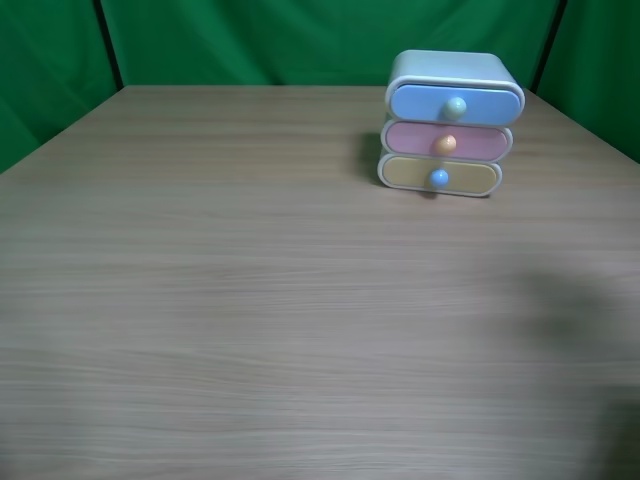}
    \caption{Block Storing}
    \label{fig:task-b}
  \end{subfigure}


  \begin{subfigure}[b]{0.32\textwidth}
    \centering
    \includegraphics[width=0.49\linewidth,keepaspectratio]{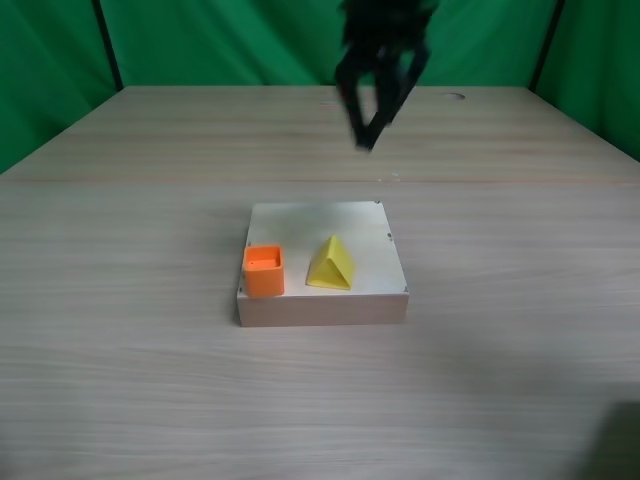}\hfill%
    \includegraphics[width=0.49\linewidth,keepaspectratio]{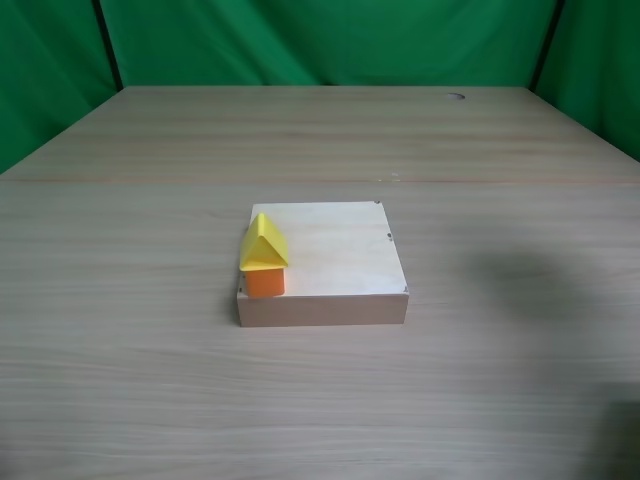}
    \caption{House Building}
    \label{fig:task-c}
  \end{subfigure}
  \hfill
  \begin{subfigure}[b]{0.32\textwidth}
    \centering
    \includegraphics[width=0.49\linewidth,keepaspectratio]{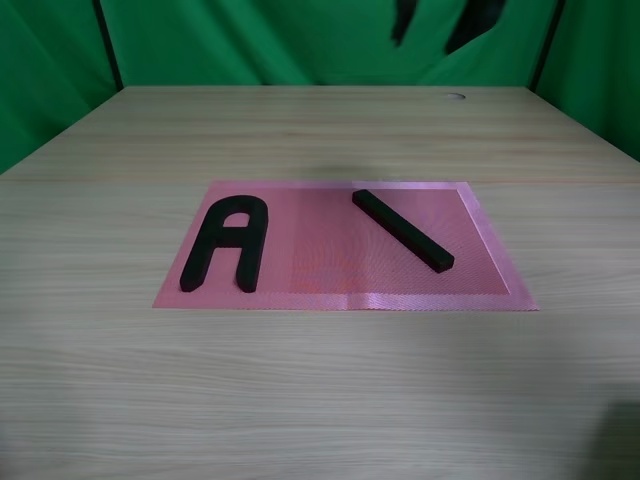}\hfill%
    \includegraphics[width=0.49\linewidth,keepaspectratio]{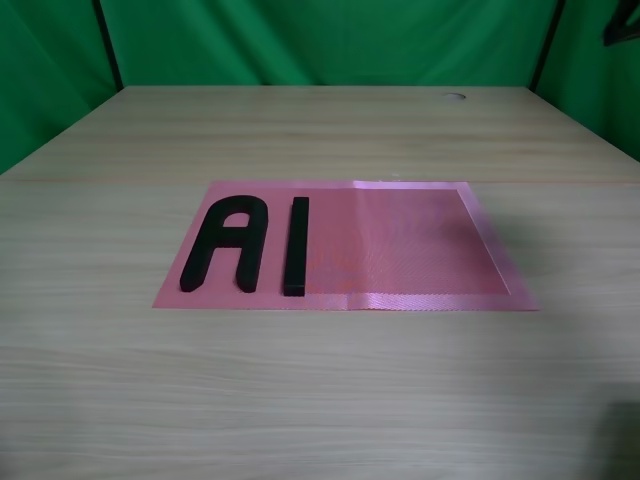}
    \caption{Letter Aligning}
    \label{fig:task-d}
  \end{subfigure}
  \hfill
  \begin{subfigure}[b]{0.32\textwidth}
    \centering
    \includegraphics[width=0.49\linewidth,keepaspectratio]{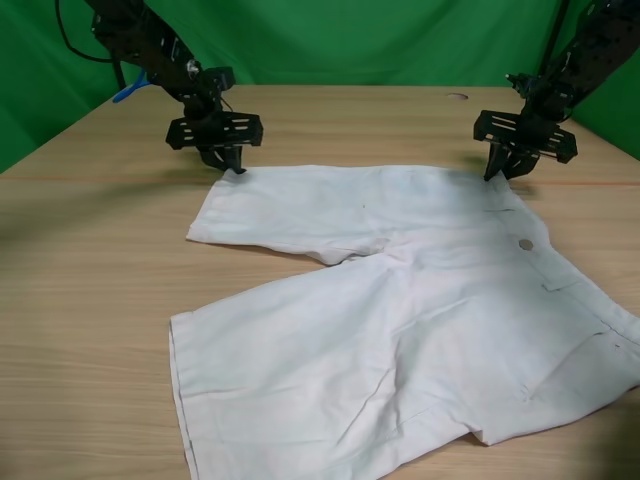}\hfill%
    \includegraphics[width=0.49\linewidth,keepaspectratio]{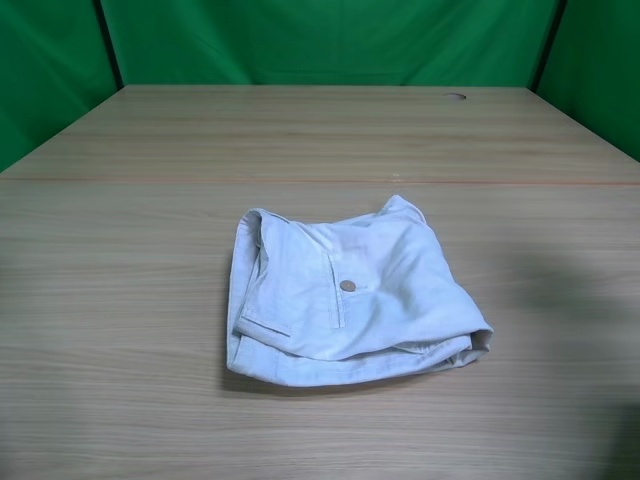}
    \caption{Shorts Folding}
    \label{fig:task-e}
  \end{subfigure}

  \caption{\textbf{Real-robot tasks.} For each task, the first and last frames show the initial and goal states, respectively. See~\Cref{app:realrobot_experimental_details} for full task descriptions.}
  \label{fig:real_robot_tasks}
\end{figure*}

\begin{table}[htbp]
    \centering
    \caption{Success rates (\%) on five manipulation tasks on a mobile ALOHA robot.}
    \label{tab:real-robot}
    \setlength{\tabcolsep}{4pt}
    \resizebox{\linewidth}{!}{
    \begin{tabular}{lccccc}
        \toprule
        Model & Banana in Pot & Block Storing & House Building & Letter Aligning & Shorts Folding \\
        \midrule
        \rowcolor{blue!8} \textbf{\equivla{}} (ours)  & \textbf{15/20} & \textbf{11/20} & \textbf{10/20} & \textbf{19/20} & \textbf{17/20} \\
        GR00T~N1.5 & 12/20 & 9/20 & 3/20 & 13/20 & \textbf{17/20} \\
        \bottomrule
    \end{tabular}
    }
\end{table}
\textbf{Results.}
\equivla{} matches or outperforms GR00T~N1.5 on all five tasks. The largest gains arise
on tasks requiring orientation-invariant grasping or placement: Letter Aligning, where
the `I' block appears at arbitrary orientations, yields $95\%$ vs.\ $65\%$ ($+30$\,pp);
House Building, where the triangular block must be grasped at the correct angle, despite
varied placement, yields $50\%$ vs.\ $15\%$ ($+35$\,pp). Conversely, Shorts Folding,
where the folding strategy is largely orientation-independent, shows identical
performance ($85\%$ for both), suggesting that the equivariant inductive bias incurs
no cost when rotational symmetry is absent. Averaged across all tasks, \equivla{}
achieves $72\%$ vs.\ $54\%$.

\section{Analysis and Ablation}
\label{sec:ablations}

\textbf{Equivariance Error.}
We validate~\Cref{thm:approx_equivariance} empirically by measuring
the equivariance error:
\begin{equation}
    \epsilon_{\mathrm{eq}} = \frac{1}{M|G|}\sum_{g,\mathbf{o}}
    \left\| \hat{\mathbf{a}}(g \cdot \mathbf{o}) -
    \rho_a(g)\hat{\mathbf{a}}(\mathbf{o})\right\|
    \label{eq:epsilon-eq}
\end{equation}
over $M{=}500$ LIBERO observations across all $g \in C_8$, where
$g \cdot \mathbf{o}$ rotates the top-down image and end-effector
pose while keeping the wrist image fixed. GR00T~N1.5
exhibits $\epsilon_{\mathrm{eq}} = 7.754 \pm 3.572$; adding
\equiactor{} reduces this by $9.3\times$ to $0.837 \pm 0.738$;
the full \equivla{} further reduces to $\mathbf{0.284 \pm 0.134}$
($27.3\times$ over baseline). Full breakdown is in
\Cref{app:empirical_equivariance}.

\textbf{Sample Efficiency.}
To assess whether the geometric inductive bias of \equivla{} translates
into data efficiency, we train three model variants: GR00T~N1.5,
GR00T~N1.5 + \equiactor{}, and \equivla{}, at three demonstration
budgets: $10\%$ (${\approx}5$ demos/task), $40\%$ (${\approx}20$ demos/task), and $100\%$ (${\approx}40$ demos/task) of the full LIBERO training set, while keeping all other hyperparameters identical~(see \Cref{app:training_details}). \equivla{} consistently
outperforms GR00T~N1.5 across all demonstration budgets, with gains of
$+1.8$, $+10.6$, and $+14.5$\,pp at $10\%$, $40\%$, and $100\%$
respectively, demonstrating that the $\SO(2)$ equivariance inductive
bias is most beneficial under limited demonstration data.

\begin{table}[htbp]
\centering
\begin{minipage}[t]{0.48\textwidth}
  \centering
  \caption{Avg. success rate on LIBERO (relative control)
           at varying demonstration budgets.}
  \label{tab:sample_efficiency}
  \setlength{\tabcolsep}{5pt}
  \renewcommand{\arraystretch}{1.08}
  \scriptsize
  \begin{tabular}{lccc}
    \toprule
    Method & 10\% & 40\% & 100\% \\
    \midrule
    GR00T N1.5                  & 58.4 & 73.9 & 78.1 \\
    GR00T N1.5 + \equiactor{}   & 58.8 & 84.1 & 91.0 \\
    \rowcolor{blue!8} \textbf{\equivla{}} (ours)        & \textbf{60.2} & \textbf{84.5} & \textbf{92.6} \\
    \bottomrule
  \end{tabular}
\end{minipage}
\hfill
\begin{minipage}[t]{0.5\textwidth}
  \centering
  \caption{Avg. success rate on LIBERO (relative control) using different symmetry groups.}
  \label{tab:ablation_group_order}
  \scriptsize
\begin{tabular}{lcccccc}
\toprule
    $C_N$ & 10 & Goal & Obj & Spat & Avg. $\uparrow$ & ms/step $\downarrow$ \\
    \midrule
    $C_4$    & 82.5 & 94.0 & 94.5 & 95.5 & 91.6 & \textbf{161} \\
    $C_8$    & \textbf{87.5} & 89.5 & 98.0 & 95.5 & 92.6 & 194 \\
    $C_{16}$ & 87.0 & \textbf{95.4} & \textbf{98.2} & \textbf{96.4} & \textbf{94.3} & 243 \\
    \bottomrule
\end{tabular}
\end{minipage}
\end{table}

\textbf{Group Choice Trade-offs.}
We ablate different choices of $C_N \in \{C_4, C_8, C_{16}\}$ for \equiactor{}'s
steerable layers across all four LIBERO suites, reporting success rates and
per-step wall-clock time on a single H100 in \autoref{tab:ablation_group_order}.
For reference, GR00T~N1.5 runs at $64$\,ms/step and GR00T~N1.5 + \equiactor{}
at $147$\,ms/step, confirming that steerable layers add only modest overhead over
the standard DiT head. The dominant cost in \equivla{} comes from
\equiperceptor{}'s $|G|$ ViT passes, adding $14$--$96$\, ms depending on group
size. $C_{16}$ achieves the highest accuracy ($94.3\%$) but at a $25\%$ latency
premium over $C_8$ for $+1.7$\,pp gain. $C_8$ offers the best
accuracy--latency trade-off ($92.6\%$, $194$\,ms/step) and is adopted as the
default group in \equivla{}.


\section{Conclusion}
\label{sec:discussion_conclusion}
\equivla{} demonstrates that $\SO(2)$ equivariance and large-scale
VLA pretraining are complementary when geometric structure is injected
at the right points in the pipeline. By symmetrizing the visual adapter
and action head while leaving the frozen VLM unmodified, we preserve
the rich visual and semantic priors of large-scale pretraining while
gaining systematic rotational generalization. The benefit scales with
data: \equivla{} outperforms the non-equivariant baseline even at
$10\%$ of demonstrations, with the gap widening at full data,
confirming consistent benefits regardless of scale. More broadly,
the results suggest that a modular approach to symmetrizing VLAs is
viable: geometric inductive biases need not be baked in from scratch
but can be composed onto existing pretrained systems. We hope this
work encourages the broader adoption of geometric structure as a
design principle for capable and data-efficient robot learning systems, especially for models with large pretrained language-vision backbones.

\textbf{Limitations.}
\emph{Symmetry scope:} we focus on $\SO(2)$ (planar rotation about the gravity axis), which is the dominant symmetry for tabletop manipulation with a top-down camera. Tasks involving out-of-plane rotations (e.g., wall-mounted manipulation) would require extending to $\SO(3)$ or $\mathrm{SE}(3)$, and since \texttt{escnn} supports
$\SO(3)$ steerable layers, adapting \equiactor{} to 3D symmetries
is a natural next step, though the increased representational
complexity of 3D irreducible representations warrants careful design.
\emph{Computational cost:} Token-level Frame Averaging requires 
$|G|{=}8$ forward passes through the frozen ViT, increasing 
\equivla{}'s inference latency to $194$\,ms/step vs.\ 
$64$\,ms/step for GR00T~N1.5; future work could explore stochastic frame
averaging~\citep{duval2023faenet}, which samples a random subset of
$G$ per step, to reduce this overhead while approximately preserving
equivariance guarantees.
\emph{Action head from scratch:} \equiactor{} cannot reuse 
pretrained DiT weights due to the structural constraints of 
steerable layers; distilling pretrained weights into equivariant 
architectures via knowledge distillation or equivariant fine-tuning 
is an important open problem.



\clearpage

\bibliography{references}

\clearpage
\appendix
\renewcommand\thesection{\Alph{section}}
\crefalias{section}{appendix}
\crefalias{subsection}{appendix}
\section*{Appendix Contents}
\startcontents[appendix]
\printcontents[appendix]{l}{1}{\setcounter{tocdepth}{2}}
\clearpage

\section{Real-Robot Environment Details}
\label{app:realrobot_experimental_details}

We evaluate \equivla{} on five real-robot manipulation tasks using the Mobile ALOHA platform~\citep{fu2024mobile}, a bimanual mobile manipulation system with two ViperX 6-DoF arms. Each task uses a top-down RGB camera and a wrist-mounted RGB camera ($640 \times 480$ resolution each) as visual inputs. Four tasks use only the right arm, while Shorts Folding requires bimanual manipulation. All experiments are conducted on a single workstation equipped with an NVIDIA RTX 4070 GPU. Each task is trained from $150$ teleoperated demonstrations. Representative rollout sequences are shown in~\Cref{fig:real_robot_rollout_gallery}. Task definitions and environment randomization details are provided below.

\begin{figure*}[ht]
\centering
\setlength{\tabcolsep}{1pt}
\renewcommand{\arraystretch}{0.8}
\begin{tabular}{c c c c c c}
\includegraphics[width=0.155\linewidth]{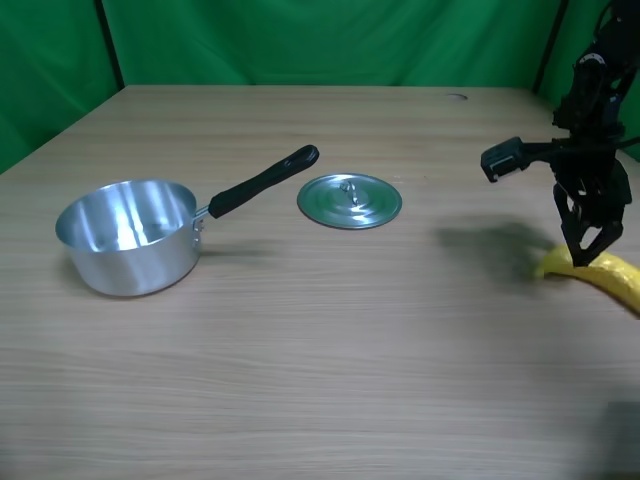} &
\includegraphics[width=0.155\linewidth]{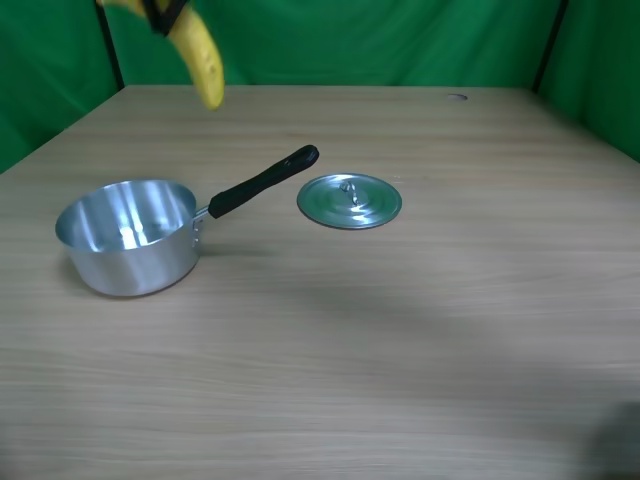} &
\includegraphics[width=0.155\linewidth]{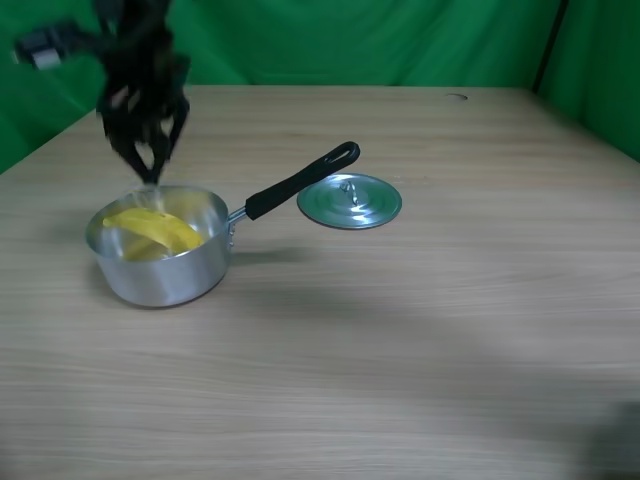} &
\includegraphics[width=0.155\linewidth]{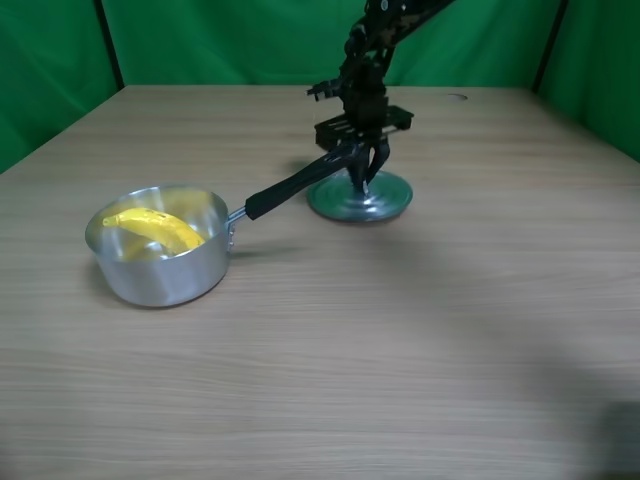} &
\includegraphics[width=0.155\linewidth]{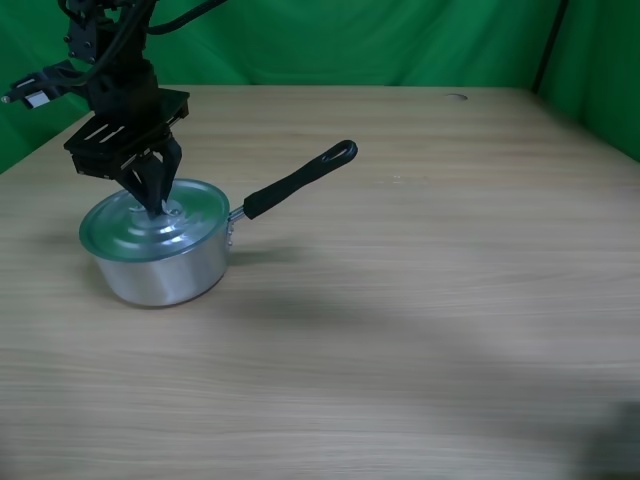} &
\includegraphics[width=0.155\linewidth]{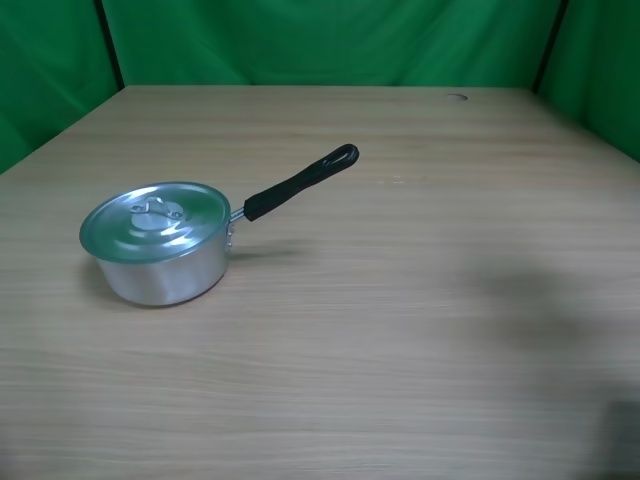} \\
\multicolumn{6}{c}{\small Banana in Pot} \\[2pt]

    \includegraphics[width=0.155\linewidth]{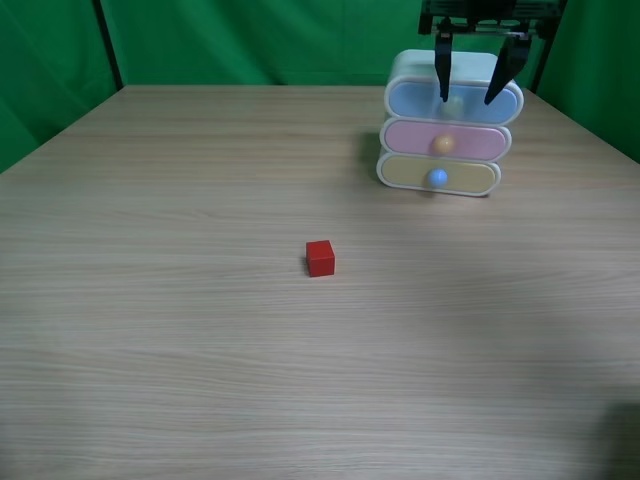} &
    \includegraphics[width=0.155\linewidth]{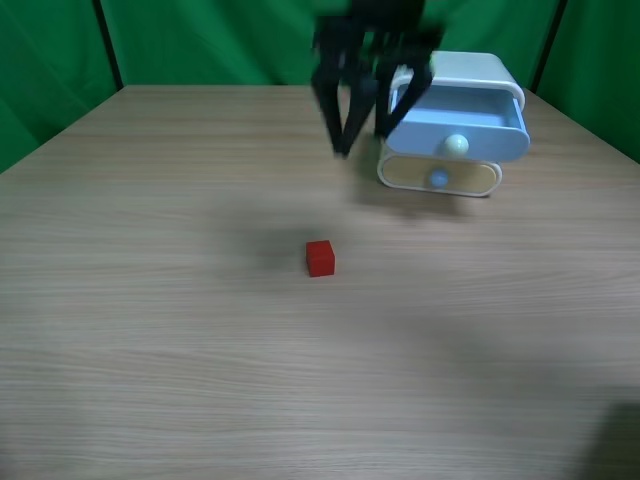} &
    \includegraphics[width=0.155\linewidth]{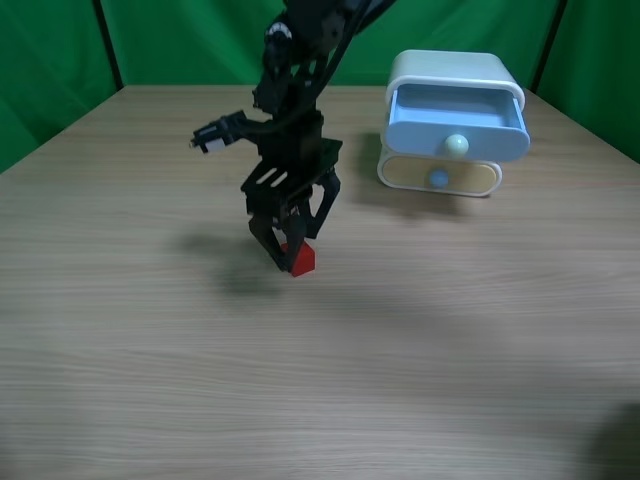} &
    \includegraphics[width=0.155\linewidth]{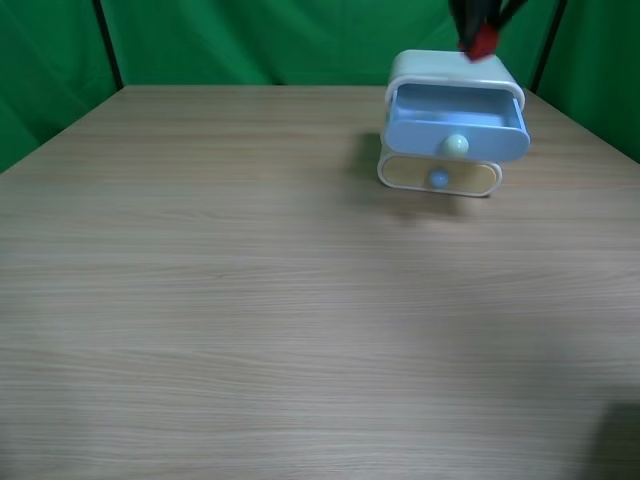} &
    \includegraphics[width=0.155\linewidth]{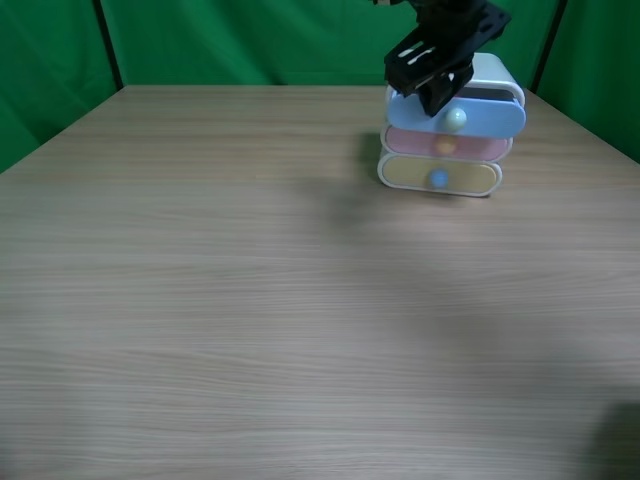} &
    \includegraphics[width=0.155\linewidth]{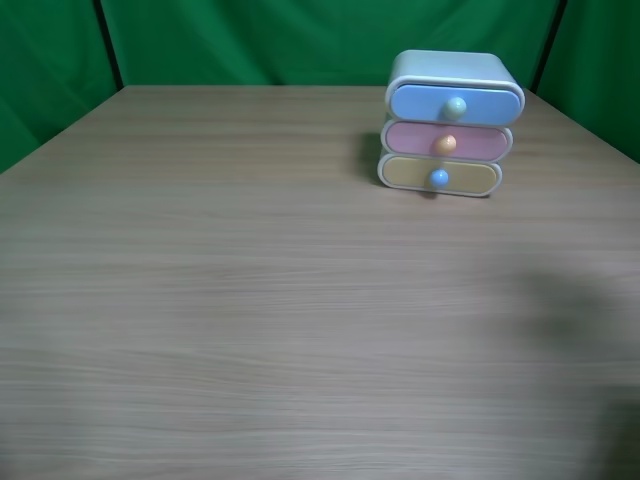} \\
    \multicolumn{6}{c}{\small Block Storing} \\[2pt]

    \includegraphics[width=0.155\linewidth]{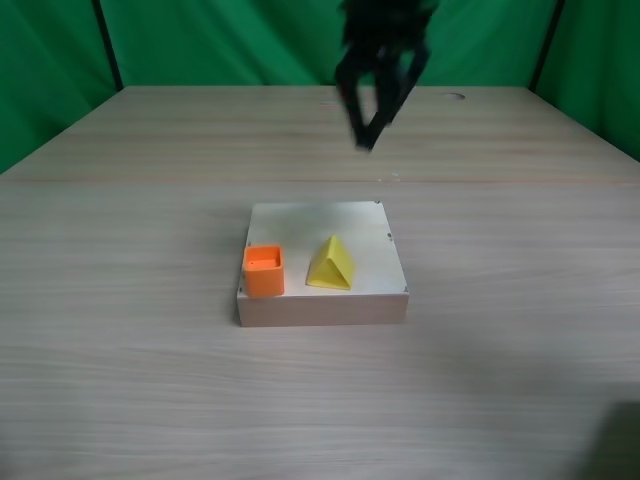} &
    \includegraphics[width=0.155\linewidth]{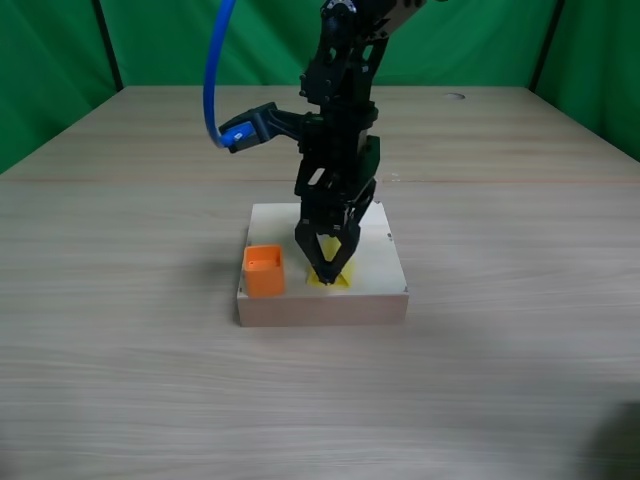} &
    \includegraphics[width=0.155\linewidth]{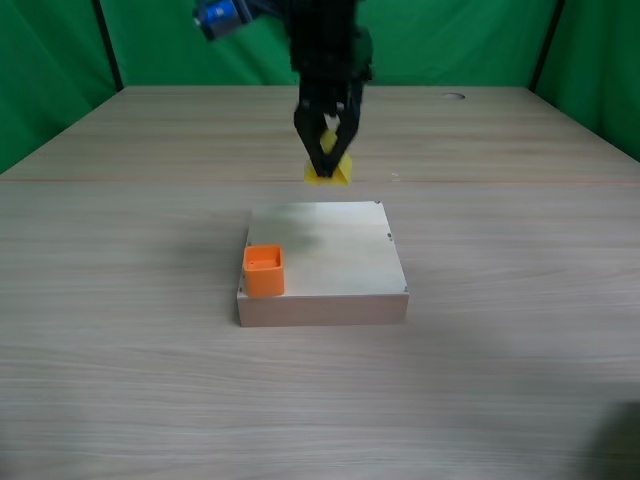} &
    \includegraphics[width=0.155\linewidth]{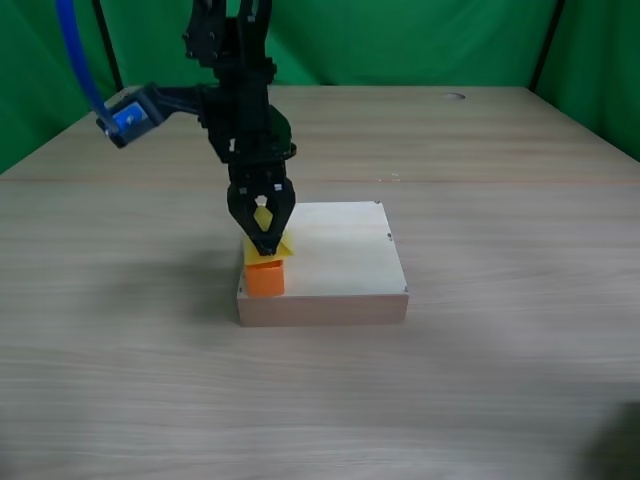} &
    \includegraphics[width=0.155\linewidth]{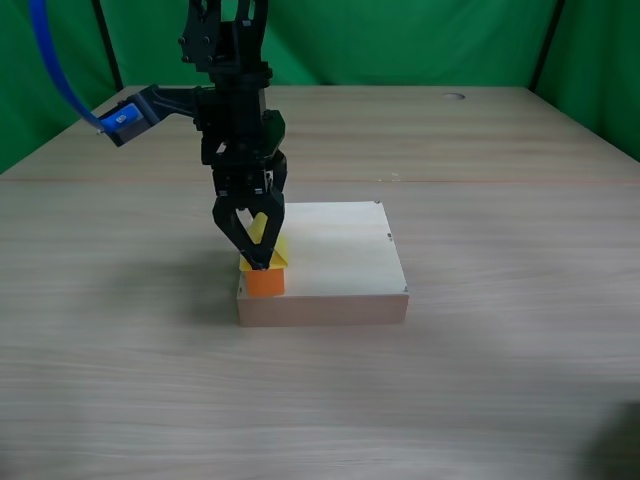} &
    \includegraphics[width=0.155\linewidth]{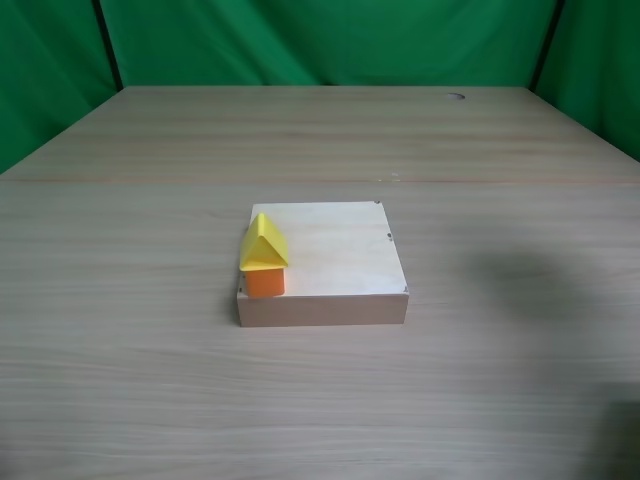} \\
    \multicolumn{6}{c}{\small House Building} \\[2pt]

    \includegraphics[width=0.155\linewidth]{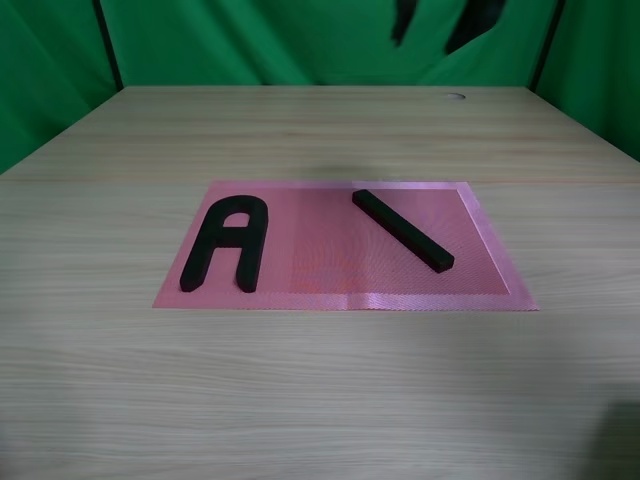} &
    \includegraphics[width=0.155\linewidth]{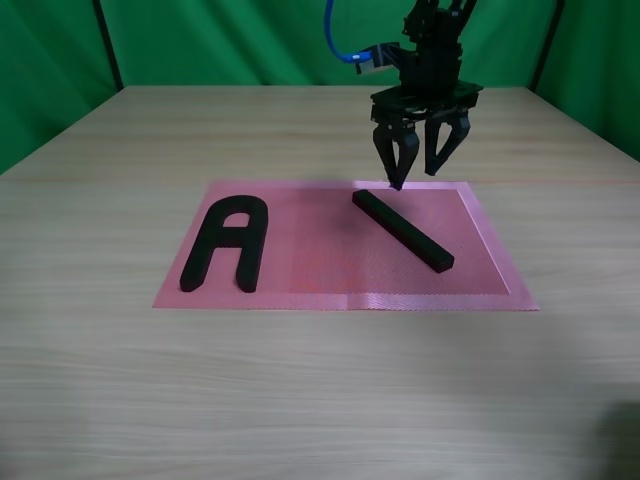} &
    \includegraphics[width=0.155\linewidth]{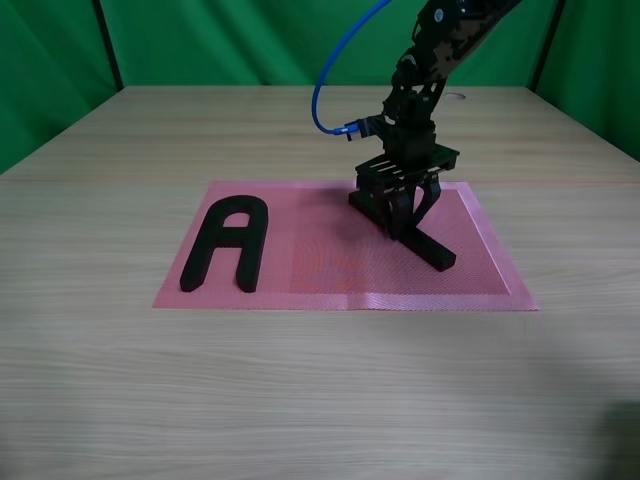} &
    \includegraphics[width=0.155\linewidth]{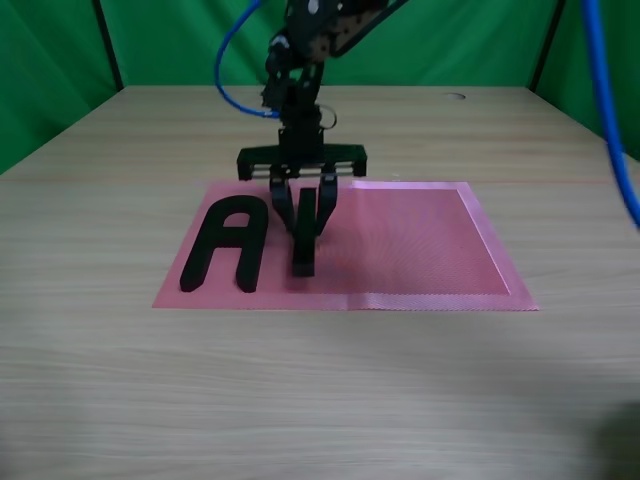} &
    \includegraphics[width=0.155\linewidth]{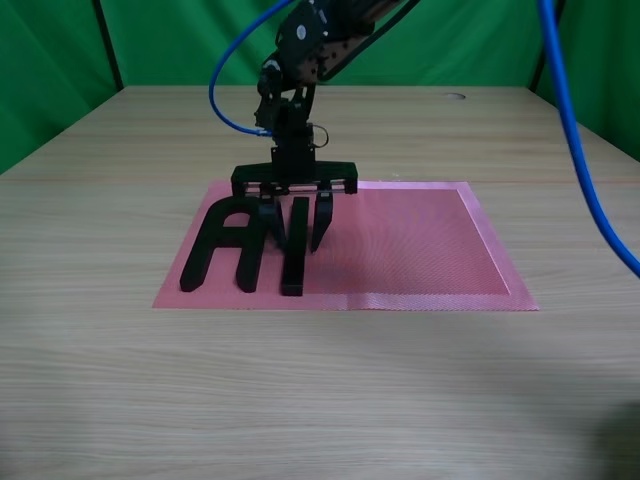} &
    \includegraphics[width=0.155\linewidth]{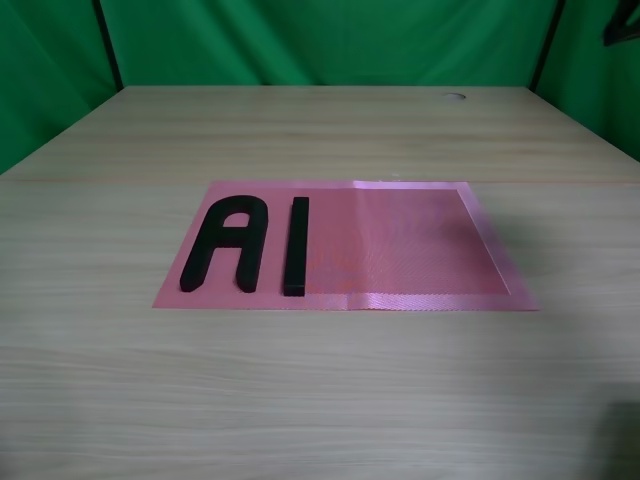} \\
    \multicolumn{6}{c}{\small Letter Aligning} \\[2pt]

    \includegraphics[width=0.155\linewidth]{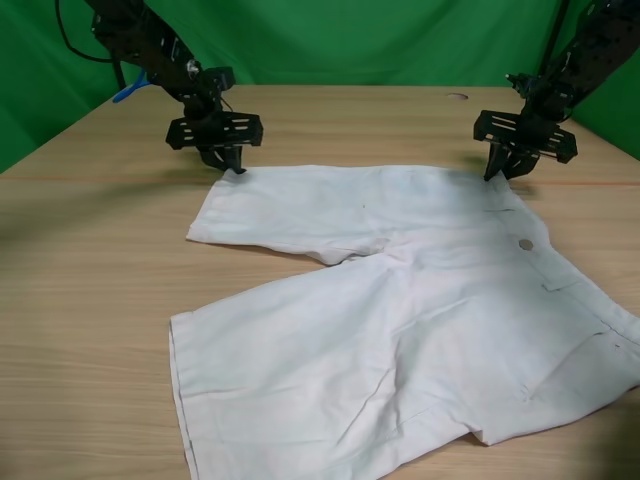} &
    \includegraphics[width=0.155\linewidth]{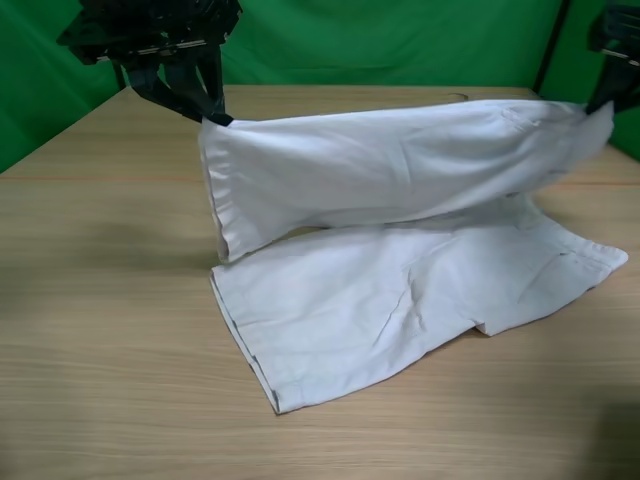} &
    \includegraphics[width=0.155\linewidth]{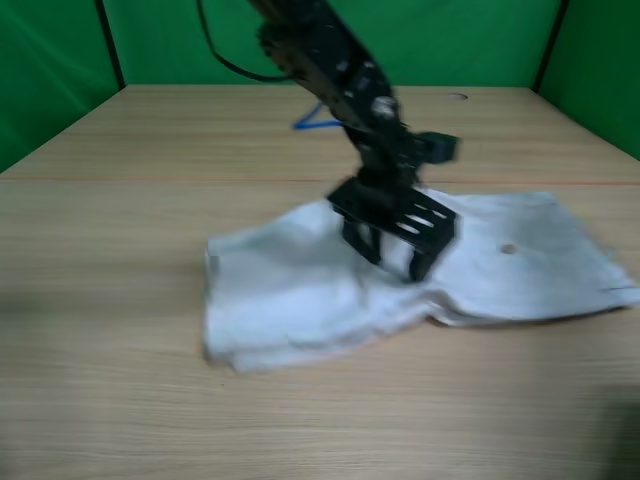} &
    \includegraphics[width=0.155\linewidth]{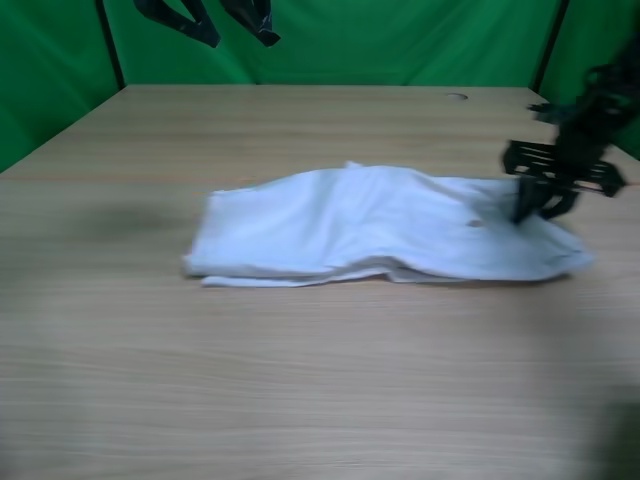} &
    \includegraphics[width=0.155\linewidth]{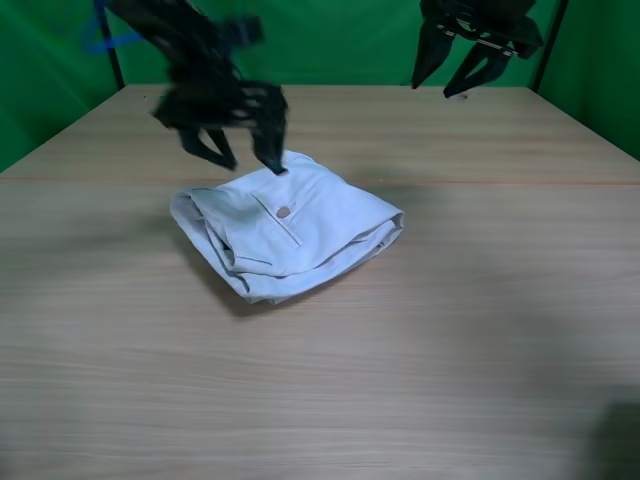} &
    \includegraphics[width=0.155\linewidth]{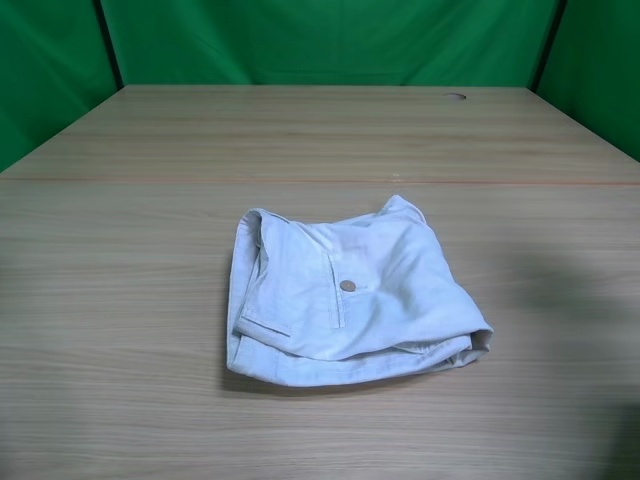} \\
    \multicolumn{6}{c}{\small Shorts Folding}
\end{tabular}
\caption{
Representative real-robot rollouts on the five Mobile ALOHA tasks.
Each row shows six keyframes from
initialization to task completion. The figures illustrate the
diversity of manipulation skills evaluated in our real-world benchmark,
including object placement, drawer interaction, precise orientation
alignment, stacking, and bimanual deformable-object manipulation.
}
\label{fig:real_robot_rollout_gallery}

\end{figure*}

\textbf{Banana in Pot:} The robot picks up a banana and places it into a pot, then grasps the pot lid and covers the pot. This task requires reasoning about spatial relationships among multiple objects, while exploiting the natural rotational symmetry of the pot and lid about the vertical axis. The banana, pot, and lid are initialized with random positions and orientations throughout the workspace.

\textbf{Block Storing:} The robot opens a drawer, picks up a square block, and places it inside the drawer, then closes the drawer. This sequential three-phase task requires precise interaction with a constrained target and represents a high-equivariant regime, where both the drawer and block can appear at varied orientations. The drawer and block are randomly placed at arbitrary positions within the workspace. The block is additionally randomized over full $360^\circ$ rotations.

\textbf{House Building:} The robot picks up the triangular block and stacks it on the square block to form a house-like structure. Success requires precise gripper orientation to grasp each triangular block at the correct angle, as the blocks appear at varied orientations and will slip if grasped incorrectly. Both the triangular and square blocks are randomly placed and oriented ($360^\circ$) on a square wooden platform.

\textbf{Letter Alignment:} A letter ``A'' block is placed at 
one of the four corners of the table, and a letter ``I'' block is 
placed at a random position and orientation elsewhere on the table. 
The robot must pick up the ``I'' block and align it next to the ``A'' 
block to form the word ``AI''. Since the ``I'' block appears at 
arbitrary orientations and must be precisely oriented relative to the 
fixed ``A'', this task represents the strongest test of rotational 
generalization in our real-robot suite.

\textbf{Shorts Folding (bimanual):} The robot picks up a pair of shorts 
from the table and folds them in half. This task involves deformable object manipulation with minimal rotational variation, as the folding strategy is largely independent of the shorts' initial orientation. The shorts are placed at the center of the table with small position perturbations ($\pm10$\,cm) and orientation perturbations ($\pm10^\circ$).

\section{Training Details}
\label{app:training_details}
\textbf{Flow-Matching Objective.}
\equiactor{} is trained with the conditional flow-matching objective~\citep{bjorck2025gr00t}. Given a demonstration action chunk $\mathbf{a}_t$, a noise sample $\boldsymbol{\epsilon} \sim \mathcal{N}(\mathbf{0}, \mathbf{I})$, and a flow time $k \sim \mathcal{U}[0,1]$, the noisy action is $\mathbf{a}^k_t = (1-k)\boldsymbol{\epsilon} + k\,\mathbf{a}_t$. The model predicts the velocity field $\mathbf{v}_{\theta}(\mathbf{a}^k_t, k, \mathbf{z}^{\mathrm{out}}, \mathbf{z}^s_t)$ and is trained to minimize:
\begin{equation}
    \mathcal{L} = \mathbb{E}_{k,\boldsymbol{\epsilon},\mathbf{a}_t}\left[
      \left\lVert \mathbf{v}_{\theta}(\mathbf{a}^k_t, k, \mathbf{z}^{\mathrm{out}}, \mathbf{z}^s_t)
      - (\mathbf{a}_t - \boldsymbol{\epsilon}) \right\rVert^2
    \right].
    \label{eq:flow_loss}
\end{equation}
At inference, actions are generated by integrating the learned velocity field from $\boldsymbol{\epsilon}$ to $\hat{\mathbf{a}}_t$ using an ODE solver.~The flow-matching objective is $G$-invariant and the ODE integration preserves equivariance, with a proof shown in~\Cref{prop:flow_equivariance}.

\textbf{Hyperparameters.}
We train all three variants (GR00T~N1.5, GR00T~N1.5 + \equiactor{}, 
and \equivla{}) with identical hyperparameters for each benchmark to 
ensure fair comparison. For LIBERO, the batch size and learning rate are set to $64$ and $1\times10^{-4}$. Training steps scale with the demonstration budget: $30$k steps for $100\%$ data, $15$k steps for $40\%$, and $10$k steps for $10\%$. For CALVIN, the batch size and maximum training steps are 64 and 120k steps, with the same learning rate. For real-robot tasks on Mobile ALOHA, models are fine-tuned from the publicly released GR00T~N1.5 3B checkpoint~\citep{bjorck2025gr00t} with batch size 32, learning rate $1\times10^{-4}$, and 150k steps. At inference, actions are decoded via Euler integration with $K{=}8$ denoising steps for simulation benchmarks and $K{=}4$ for real-robot tasks. All three variants are trained on an NVIDIA H100 GPU.

\textbf{Training Time.} GR00T~N1.5 takes approximately 5 hours for LIBERO and 19 hours for CALVIN; GR00T~N1.5 + \equiactor{} takes 11 and 43 hours respectively; \equivla{} takes 17 and 63 hours. For real-robot fine-tuning, GR00T~N1.5 and \equivla{} take approximately 18 and 60 hours, respectively.

\textbf{Architecture Details.}
\equiactor{} consists of $N{=}16$ equivariant DiT layers, each comprising an equivariant cross-attention block and an equivariant self-attention block. All linear projections are replaced by $G$-steerable layers operating in regular feature space: the hidden dimension $d$ is factored as $d = n_{\mathrm{channels}} \times |G|$, where $n_{\mathrm{channels}}{=}192$ is the number of regular feature channels and $|G|{=}8$ for $C_8$. The number of attention heads is inherited from GR00T~N1.5 and kept unchanged. The equivariant state encoder, action encoder, and action decoder share the same regular feature space of dimension $d$.

\section{Equivariance Analysis}
\label{app:equivariance_proof}

We provide a theoretical analysis of the equivariance guarantees of token-level Frame Averaging~(\Cref{sec:method:equiperceptor}), the Equivariant Adapter, and the end-to-end \equivla{} pipeline, deriving exact conditions under which equivariance holds, bounding the approximation error when it does not, and connecting the bound to sample complexity.

\textbf{Notation.}\;
Let $G = C_u$ be a cyclic group of $u$ planar rotations acting on images via $g \cdot \mathbf{x}$.
Let $f_\theta\!: \mathbb{R}^{H \times W \times 3} \to \mathbb{R}^{N \times D}$ be the frozen ViT producing $N$ patch tokens of dimension~$D$.
Define the \emph{output action} $A(h) = \tau(h) \otimes \rho_{\mathrm{reg}}(h)$, where $\tau(h) \in \mathrm{Perm}(N)$ is the spatial token permutation and $\rho_{\mathrm{reg}}(h)$ is the regular representation of~$G$.
The tensor product acts jointly on token positions ($\tau$) and feature channels ($\rho_{\mathrm{reg}}$), mapping $\mathbb{R}^{N \times D}$ into $\mathbb{R}^{N \times D \times |G|}$.
When $\tau$ is approximated by nearest-neighbor grid reassignment, we write $\tilde{\tau}$ and $\tilde{A}(h) = \tilde{\tau}(h) \otimes \rho_{\mathrm{reg}}(h)$.
We denote the full observation (camera images and proprioceptive state) by $\mathbf{o}$ and its image component by $\mathbf{x}$.

\subsection{Exact Equivariance}

\begin{proposition}[Exact Token-Level FA Equivariance]
\label{prop:exact_equivariance}
Suppose $\tau: G \to \mathrm{Perm}(N)$ is a group homomorphism, i.e., $\tau(gh) = \tau(g)\circ\tau(h)$ for all $g, h \in G$.
Then the token-level Frame Average
\begin{equation}
    \mathbf{z}^{\mathrm{eq}}(\mathbf{x})
    = \frac{1}{|G|}\sum_{h \in G} A(h^{-1}) \cdot f_\theta(h \cdot \mathbf{x})
\end{equation}
is exactly $G$-equivariant:
$\mathbf{z}^{\mathrm{eq}}(g \cdot \mathbf{x}) = A(g) \cdot \mathbf{z}^{\mathrm{eq}}(\mathbf{x})$
for all $g \in G$, regardless of the choice of~$f_\theta$.
\end{proposition}

\begin{proof}
\begin{align}
    \mathbf{z}^{\mathrm{eq}}(g \cdot \mathbf{x})
    &= \frac{1}{|G|}\sum_{h \in G} A(h^{-1}) \cdot f_\theta(hg \cdot \mathbf{x})
    \overset{h'=hg}{=}
    \frac{1}{|G|}\sum_{h'} A\bigl(g(h')^{-1}\bigr) \cdot f_\theta(h' \cdot \mathbf{x}) \nonumber\\
    &\overset{(\star)}{=}
    A(g)\,\frac{1}{|G|}\sum_{h'} A\bigl((h')^{-1}\bigr) \cdot f_\theta(h' \cdot \mathbf{x})
    = A(g) \cdot \mathbf{z}^{\mathrm{eq}}(\mathbf{x}),
    \label{eq:exact_fa_proof}
\end{align}
where $(\star)$ uses $A(g(h')^{-1}) = A(g)\,A((h')^{-1})$, which holds because $\tau$ is a homomorphism and $\rho_{\mathrm{reg}}$ is a group representation, so $A = \tau \otimes \rho_{\mathrm{reg}}$ is a homomorphism.
\end{proof}

\textbf{Invariance of $\mathbf{z}^{\mathrm{inv}}$.}
The invariant stream $\mathbf{z}^{\mathrm{inv}}(\mathbf{x}) = \frac{1}{|G|}\sum_{h \in G} f_\theta(h \cdot \mathbf{x})$ is exactly $G$-invariant for any~$f_\theta$, by the standard change-of-variable $h' = hg$:
$\mathbf{z}^{\mathrm{inv}}(g \cdot \mathbf{x})
= \frac{1}{|G|}\sum_h f_\theta(hg \cdot \mathbf{x})
= \frac{1}{|G|}\sum_{h'} f_\theta(h' \cdot \mathbf{x})
= \mathbf{z}^{\mathrm{inv}}(\mathbf{x})$~\citep{puny2021frame}.

\subsection{Approximate Equivariance Bound}

When $\tilde{\tau}$ is not a group homomorphism (as occurs for any 
$C_u$ with $u \nmid 4$ on a square patch grid), the Frame Average 
achieves only approximate equivariance. We introduce the 
\emph{representation defect} to quantify this gap.

\begin{definition}[Representation Defect]
\label{def:defect}
For the approximate output action $\tilde{A}(h) = \tilde{\tau}(h) \otimes \rho_{\mathrm{reg}}(h)$, define:
\begin{equation}
    D(g, h) \;\coloneqq\; \tilde{A}(gh) - \tilde{A}(g)\,\tilde{A}(h),
    \qquad g, h \in G.
\end{equation}
Since $\rho_{\mathrm{reg}}$ is an exact (unitary) representation, the defect reduces to the spatial component:
\begin{equation}
    \norm{D(g, h)}_{\mathrm{op}}
    = \norm{\tilde{\tau}(gh) - \tilde{\tau}(g)\circ\tilde{\tau}(h)}_{\mathrm{op}}.
    \label{eq:defect_spatial}
\end{equation}
The \emph{maximum defect} is $\Delta \coloneqq \max_{g,h \in G} \norm{D(g,h)}_{\mathrm{op}}$.
\end{definition}

\begin{theorem}[Approximate FA Equivariance Bound]
\label{thm:approx_equivariance}
Let $\tilde{\mathbf{z}}^{\mathrm{eq}}$ be the token-level FA with approximate spatial reassignment $\tilde{\tau}$. Then for all $g \in G$:
\begin{equation}
    \tilde{\mathbf{z}}^{\mathrm{eq}}(g \cdot \mathbf{x})
    = \tilde{A}(g) \cdot \tilde{\mathbf{z}}^{\mathrm{eq}}(\mathbf{x})
    + \boldsymbol{\varepsilon}(g, \mathbf{x}),
\end{equation}
where the equivariance residual satisfies:
\begin{equation}
    \norm{\boldsymbol{\varepsilon}(g, \mathbf{x})}
    \;\leq\;
    \frac{1}{|G|}\sum_{h \in G} \norm{D(g, h^{-1})}_{\mathrm{op}} \cdot \norm{f_\theta(h \cdot \mathbf{x})}
    \;\leq\;
    \Delta \cdot B(\mathbf{x}),
    \label{eq:approx_bound}
\end{equation}
with $B(\mathbf{x}) = \max_{h \in G} \norm{f_\theta(h \cdot \mathbf{x})}$.
In particular, $\Delta = 0$ implies exact equivariance.
\end{theorem}

\begin{proof}
Following the same substitution as~\Cref{eq:exact_fa_proof}:
\begin{align}
    \tilde{\mathbf{z}}^{\mathrm{eq}}(g \cdot \mathbf{x})
    &= \frac{1}{|G|}\sum_{h'} \tilde{A}\bigl(g(h')^{-1}\bigr) \cdot f_\theta(h' \cdot \mathbf{x}).
\end{align}
Decomposing via the defect $\tilde{A}(g(h')^{-1}) = \tilde{A}(g)\,\tilde{A}((h')^{-1}) + D(g, (h')^{-1})$:
\begin{align}
    &= \tilde{A}(g)\cdot
    \underbrace{\frac{1}{|G|}\sum_{h'}\tilde{A}\bigl((h')^{-1}\bigr) \cdot f_\theta(h' \cdot \mathbf{x})}_{= \;\tilde{\mathbf{z}}^{\mathrm{eq}}(\mathbf{x})}
    \;+\;
    \underbrace{\frac{1}{|G|}\sum_{h'} D\bigl(g, (h')^{-1}\bigr) \cdot f_\theta(h' \cdot \mathbf{x})}_{= \;\boldsymbol{\varepsilon}(g, \mathbf{x})}.
\end{align}
The first bound follows from the triangle inequality.
The second uses $\norm{D(g, h^{-1})}_{\mathrm{op}} \leq \Delta$ and $\norm{f_\theta(h \cdot \mathbf{x})} \leq B(\mathbf{x})$.
\end{proof}

\subsection{Geometric Characterization of the Defect}

We characterize when the representation defect vanishes, bound its magnitude for the groups used in practice, and formally connect it to the spatial displacement error.

\begin{definition}[Maximum Spatial Displacement]
\label{def:displacement}
For $h \in G$ and patch center $\mathbf{p}_i$ on an $n \times n$ grid with spacing~$p$, define the spatial displacement
$\delta_h(i) = \norm{\mathbf{p}_{\tilde{\tau}(h)(i)} - h \cdot \mathbf{p}_i}_2$
and the maximum displacement
$\delta_{\max} = \max_{h \in G,\; i \in [N]} \delta_h(i)$.
For interior patches whose rotated positions remain within the grid's convex hull, $\delta_h(i) \leq p\sqrt{2}/2$ (half-diagonal of a grid cell).
Near boundaries, $\delta_h(i)$ can exceed this bound when rotated positions fall outside the grid~(see~\Cref{tab:defect_values}).
\end{definition}

\begin{proposition}[Grid-Aligned Rotations]
\label{prop:exact_cases}
Let $G = C_u$ act on a square $n \times n$ patch grid centered at the image center.
\begin{enumerate}[label=\textup{(\alph*)},leftmargin=*,itemsep=1pt]
    \item \textbf{Exact case.}\;
    $\delta_{h}(i) = 0\;\forall\, h, i \;\Longrightarrow\; \Delta = 0$ (exact equivariance).
    On a square grid, this holds when $u \mid 4$, since $90^\circ$ rotations map patch centers exactly onto other patch centers.
    \item \textbf{$C_4$ subgroup.}\;
    For $C_8$, the $C_4$ subgroup $\set{0^\circ, 90^\circ, 180^\circ, 270^\circ}$ satisfies $\delta_h = 0$ for all positions and all $h \in C_4$.
    The defect arises solely from the $45^\circ$-offset coset $\set{45^\circ, 135^\circ, 225^\circ, 315^\circ}$.
    \item \textbf{Coset structure.}\;
    For $g \in C_4 \leq C_8$, the defect $D(g, h)$ vanishes whenever $h \in C_4$, since both $\tilde{\tau}(g)$ and $\tilde{\tau}(h)$ are exact permutations and $\tilde{\tau}(gh)$ is exact by closure.
    Thus at most $|G \setminus C_4| = 4$ terms in the sum~\eqref{eq:approx_bound} are non-zero, yielding the refined bound:
    \begin{equation}
        \norm{\boldsymbol{\varepsilon}(g, \mathbf{x})}
        \leq \frac{|G \setminus C_4|}{|G|} \cdot \Delta \cdot B(\mathbf{x})
        = \tfrac{1}{2}\,\Delta \cdot B(\mathbf{x}),
        \quad g \in C_4.
        \label{eq:refined_bound}
    \end{equation}
    \item \textbf{Permutation-norm bound.}\;
    Since $\tilde{\tau}(gh)$ and $\tilde{\tau}(g)\circ\tilde{\tau}(h)$ are both permutation matrices, their difference satisfies $\norm{D(g,h)}_{\mathrm{op}} \leq 2$ universally, and $\norm{D(g,h)}_F = \sqrt{2\,m(g,h)}$ where $m(g,h) = \bigl|\set{i : [\tilde{\tau}(g)\circ\tilde{\tau}(h)](i) \neq \tilde{\tau}(gh)(i)}\bigr|$ counts the number of positions where the composition of nearest-neighbor reassignments disagrees with the nearest-neighbor of the composed rotation.
    \item \textbf{Displacement--defect connection.}\;
    For any $g, h \in G$ and position $i \in [N]$, a disagreement
    $[\tilde{\tau}(g) \circ \tilde{\tau}(h)](i) \neq \tilde{\tau}(gh)(i)$
    can occur only if
    \begin{equation}
        \delta_h(i) + \delta_g\!\bigl(\tilde{\tau}(h)(i)\bigr) + \delta_{gh}(i) \;\geq\; p,
        \label{eq:displacement_defect}
    \end{equation}
    where $p$ is the grid spacing. Consequently $m(g,h) \leq \bigl|\bigl\{i : \delta_h(i) + \delta_g(\tilde{\tau}(h)(i)) + \delta_{gh}(i) \geq p \bigr\}\bigr|$.
\end{enumerate}
\end{proposition}

\begin{proof}
(a)~When all displacements vanish, every rotation maps each patch center exactly onto another, so $\tilde{\tau}(h)$ is the unique permutation that tracks $h$'s action on the grid.
Uniqueness forces $\tilde{\tau}(gh)(i) = [\tilde{\tau}(g) \circ \tilde{\tau}(h)](i)$ for all $i$, hence $\tilde{\tau}$ is a homomorphism and $\Delta = 0$.
On a square grid, the $90^\circ$ rotation $(x,y) \mapsto (-y, x)$ is a bijection on integer coordinates; all multiples of $90^\circ$ compose to grid bijections, so $C_4$ has $\delta_{\max} = 0$. In contrast, a $45^\circ$ rotation maps $(x,y) \mapsto \bigl(\frac{x-y}{\sqrt{2}}, \frac{x+y}{\sqrt{2}}\bigr)$, which generically leaves the integer lattice.

(b)~Follows from~(a) applied to the $C_4$ subgroup.

(c)~For $g, h \in C_4$: $\tilde{\tau}(g)$ and $\tilde{\tau}(h)$ are exact homomorphisms, and $gh \in C_4$ by closure, so $\tilde{\tau}(gh) = \tilde{\tau}(g) \circ \tilde{\tau}(h)$ and $D(g, h) = 0$.
The bound~\eqref{eq:refined_bound} follows by restricting the sum in~\eqref{eq:approx_bound} to $h \notin C_4$.

(d)~$\tilde{\tau}(gh)$ and $\tilde{\tau}(g)\circ\tilde{\tau}(h)$ are $N \times N$ permutation matrices.
Each position where they disagree contributes exactly $2$ to the squared Frobenius norm (two mismatched unit entries per row), hence $\norm{\cdot}_F = \sqrt{2m}$.
The operator norm bound $\leq 2$ follows from the triangle inequality on permutation operator norms.

(e)~The composed path $\tilde{\tau}(g) \circ \tilde{\tau}(h)$ maps $i$ to the nearest grid point to $g \cdot \mathbf{p}_{\tilde{\tau}(h)(i)}$, while $\tilde{\tau}(gh)$ maps $i$ to the nearest grid point to $(gh) \cdot \mathbf{p}_i$.
Since $g$ is an isometry:
\[
    \norm{g \cdot \mathbf{p}_{\tilde{\tau}(h)(i)} - (gh) \cdot \mathbf{p}_i}
    = \norm{\mathbf{p}_{\tilde{\tau}(h)(i)} - h \cdot \mathbf{p}_i}
    = \delta_h(i).
\]
By the triangle inequality, the distance between the two target grid points satisfies
\[
    \norm{\mathbf{p}_{[\tilde{\tau}(g)\circ\tilde{\tau}(h)](i)} - \mathbf{p}_{\tilde{\tau}(gh)(i)}}
    \;\leq\; \delta_h(i) + \delta_g\!\bigl(\tilde{\tau}(h)(i)\bigr) + \delta_{gh}(i).
\]
If the two target grid points are distinct, their distance is at least $p$ (the grid spacing), yielding~\eqref{eq:displacement_defect} by contraposition.
\end{proof}

\textbf{Computed defect values.}
\Cref{tab:defect_values} reports the representation defect for $C_8$ on standard ViT patch grids.
For all grid sizes, the $C_4$ subgroup has zero defect (confirming~\Cref{prop:exact_cases}(a--c)), while the $45^\circ$-offset coset contributes all non-zero terms.
The fraction of mismatched positions stabilizes around $31$--$35\%$ for worst-case $(g, h)$ pairs, with the mismatch count concentrated near grid boundaries where rotated patch centers exit the grid's convex hull.

\begin{table}[htbp]
\centering
\caption{Representation defect of $C_8$ on $n \times n$ patch grids (unit spacing $p = 1$). All $C_4 \times C_4$ pairs have $m = 0$; the $45^\circ$-offset coset is the sole source of non-zero defect.}
\label{tab:defect_values}
\small
\begin{tabular}{@{}ccccccc@{}}
\toprule
$n$ & $N$ & $\delta_{\max}$ & $\max\, m$ & $m / N$ & $\norm{D}_F$ & non-zero pairs \\
\midrule
14 & 196 & 2.74 & 68 & 34.7\% & 11.66 & 38/64 \\
16 & 256 & 3.15 & 80 & 31.2\% & 12.65 & 34/64 \\
24 & 576 & 4.79 & 176 & 30.6\% & 18.76 & 38/64 \\
32 & 1024 & 6.44 & 328 & 32.0\% & 25.61 & 38/64 \\
\bottomrule
\end{tabular}
\end{table}

\subsection{Component Equivariance: Adapter and Flow Matching}
\label{app:component_equivariance}

We establish the equivariance properties of the Equivariant Adapter~(\Cref{eq:adapter_out}) and the flow-matching training objective~\Cref{eq:flow_loss}, both of which feed into the end-to-end bound.

\begin{proposition}[Equivariant Adapter]
\label{prop:adapter_equivariance}
Let $\tilde{\mathbf{z}}^{\mathrm{eq}}$ be $G$-equivariant (in the regular representation), let $\mathbf{s}^{\mathrm{inv}}$, $\mathbf{s}^{\mathrm{lang}}$, $\mathbf{s}^{\mathrm{vis}}$ be $G$-invariant, let $\mathbf{W}_g$ be a $G$-equivariant linear map (regular-to-regular), and let $\mathbf{W}_s$ be a $G$-equivariant linear map.
Then the adapter output~\eqref{eq:adapter_out}
\begin{equation*}
    \mathbf{z}^{\mathrm{out}}_{\mathrm{eq}}
    = \boldsymbol{\alpha}^{\mathrm{reg}}
      \odot \mathbf{W}_s(\mathbf{s}^{\mathrm{inv}} \otimes \mathbf{1}_{|G|})
    + (\mathbf{1} - \boldsymbol{\alpha}^{\mathrm{reg}})
      \odot \mathbf{W}_g(\tilde{\mathbf{z}}^{\mathrm{eq}})
\end{equation*}
is $G$-equivariant: $\mathbf{z}^{\mathrm{out}}_{\mathrm{eq}}(g \cdot \mathbf{x}) = \rho_{\mathrm{reg}}(g) \cdot \mathbf{z}^{\mathrm{out}}_{\mathrm{eq}}(\mathbf{x})$.
\end{proposition}

\begin{proof}
Since $\boldsymbol{\alpha} = \sigma(\mathbf{W}_{\mathrm{gate}}[\mathbf{s}^{\mathrm{lang}}; \mathbf{s}^{\mathrm{vis}}; \bar{\mathbf{z}}^{\mathrm{eq}}])$ is composed entirely of invariant inputs, $\boldsymbol{\alpha}$ is $G$-invariant.
Its regular-representation tiling $\boldsymbol{\alpha}^{\mathrm{reg}} = \boldsymbol{\alpha} \otimes \mathbf{1}_{|G|}$ is invariant under $\rho_{\mathrm{reg}}(g)$, since $\rho_{\mathrm{reg}}(g)$ cyclically permutes coordinates and $\mathbf{1}_{|G|}$ is fixed by any permutation.

\emph{Semantic branch.}
$\mathbf{s}^{\mathrm{inv}} \otimes \mathbf{1}_{|G|}$ lies in regular-representation space but is $G$-invariant (constant across group channels).
Since $\mathbf{W}_s$ is equivariant, $\mathbf{W}_s(\mathbf{s}^{\mathrm{inv}} \otimes \mathbf{1}_{|G|})$ is also $G$-invariant.

\emph{Equivariant branch.}
$\mathbf{W}_g$ is $G$-equivariant: $\mathbf{W}_g(\rho_{\mathrm{reg}}(g)\tilde{\mathbf{z}}^{\mathrm{eq}}) = \rho_{\mathrm{reg}}(g) \mathbf{W}_g(\tilde{\mathbf{z}}^{\mathrm{eq}})$.

\emph{Combination.}
Since $\boldsymbol{\alpha}^{\mathrm{reg}}$ is invariant and $\rho_{\mathrm{reg}}(g)$ is a coordinate permutation, Hadamard products commute:
$\rho_{\mathrm{reg}}(g)[\boldsymbol{\alpha}^{\mathrm{reg}} \odot \mathbf{v}] = \boldsymbol{\alpha}^{\mathrm{reg}} \odot \rho_{\mathrm{reg}}(g)\mathbf{v}$ for any $\mathbf{v}$.
Applying this to both branches:
\begin{align*}
    \rho_{\mathrm{reg}}(g)\,\mathbf{z}^{\mathrm{out}}_{\mathrm{eq}}
    &= \boldsymbol{\alpha}^{\mathrm{reg}} \odot \underbrace{\rho_{\mathrm{reg}}(g) \mathbf{W}_s(\mathbf{s}^{\mathrm{inv}} \otimes \mathbf{1}_{|G|})}_{\text{invariant, unchanged}}
    + (\mathbf{1} - \boldsymbol{\alpha}^{\mathrm{reg}}) \odot \underbrace{\rho_{\mathrm{reg}}(g)\mathbf{W}_g(\tilde{\mathbf{z}}^{\mathrm{eq}})}_{ = \,\mathbf{W}_g(\rho_{\mathrm{reg}}(g)\tilde{\mathbf{z}}^{\mathrm{eq}})},
\end{align*}
which equals $\mathbf{z}^{\mathrm{out}}_{\mathrm{eq}}$ evaluated at $g \cdot \mathbf{x}$, since $\tilde{\mathbf{z}}^{\mathrm{eq}}(g \cdot \mathbf{x}) = \rho_{\mathrm{reg}}(g)\tilde{\mathbf{z}}^{\mathrm{eq}}(\mathbf{x})$ and all invariant quantities are unchanged.
\end{proof}

\begin{proposition}[Flow-Matching Compatibility with Equivariance]
\label{prop:flow_equivariance}
Let $\rho_a$ denote the $\SO(2)$ representation acting on the action space, and suppose {\normalfont\equiactor{}} implements a $G$-equivariant velocity field $\mathbf{v}_\theta$:
\begin{equation}
    \mathbf{v}_\theta\!\bigl(\rho_a(g)\mathbf{a}^k,\; k,\; \tilde{A}(g)\mathbf{z},\; \rho(g)\mathbf{z}^s\bigr)
    = \rho_a(g)\,\mathbf{v}_\theta(\mathbf{a}^k, k, \mathbf{z}, \mathbf{z}^s)
    \quad \forall\, g \in G.
    \label{eq:velocity_equivariance}
\end{equation}
Then:
\begin{enumerate}[label=\textup{(\roman*)},leftmargin=*,itemsep=1pt]
    \item The flow-matching loss~\eqref{eq:flow_loss} is $G$-invariant in distribution: for any $g \in G$,
    $\mathcal{L}(g \cdot \mathbf{o}, \rho_a(g)\mathbf{a}_t) = \mathcal{L}(\mathbf{o}, \mathbf{a}_t)$ in expectation over $\boldsymbol{\epsilon}$ and~$k$.
    \item The ODE integration $\hat{\mathbf{a}}_t = \mathrm{ODESolve}(\mathbf{v}_\theta, \boldsymbol{\epsilon}, 0 \to 1)$ satisfies $\hat{\mathbf{a}}_t(g \cdot \mathbf{o}, \rho(g)\mathbf{z}^s, \rho_a(g)\boldsymbol{\epsilon}) = \rho_a(g)\hat{\mathbf{a}}_t(\mathbf{o}, \mathbf{z}^s, \boldsymbol{\epsilon})$.
\end{enumerate}
\end{proposition}

\begin{proof}
(i)~The noisy action $\mathbf{a}^k = (1-k)\boldsymbol{\epsilon} + k\mathbf{a}_t$ transforms as $\rho_a(g)\mathbf{a}^k = (1-k)\rho_a(g)\boldsymbol{\epsilon} + k\rho_a(g)\mathbf{a}_t$ by linearity. Since $\rho_a(g)$ is orthogonal and $\boldsymbol{\epsilon} \sim \mathcal{N}(\mathbf{0}, \mathbf{I})$, we have $\rho_a(g)\boldsymbol{\epsilon} \sim \mathcal{N}(\mathbf{0}, \mathbf{I})$ (rotational invariance of the isotropic Gaussian).
The target velocity transforms as $\rho_a(g)(\mathbf{a}_t - \boldsymbol{\epsilon})$, and by~\eqref{eq:velocity_equivariance} the predicted velocity transforms identically, so the squared loss $\norm{\mathbf{v}_\theta - (\mathbf{a}_t - \boldsymbol{\epsilon})}^2$ is unchanged under the simultaneous action of $g$ (using $\norm{\rho_a(g)\mathbf{v}} = \norm{\mathbf{v}}$).

(ii)~At each ODE step $\mathbf{a}^{k+\Delta k} = \mathbf{a}^k + \Delta k \cdot \mathbf{v}_\theta(\mathbf{a}^k, k, \cdots)$, applying $\rho_a(g)$ to both sides and using~\eqref{eq:velocity_equivariance} shows $\rho_a(g)\mathbf{a}^{k+\Delta k} = \rho_a(g)\mathbf{a}^k + \Delta k \cdot \mathbf{v}_\theta(\rho_a(g)\mathbf{a}^k, k, \cdots)$.
By induction over integration steps, equivariance propagates from the initial noise to the final prediction~\citep{tie2025etseed}.
\end{proof}

\subsection{End-to-End Pipeline Bound}
\label{app:e2e_bound}

We compose the approximate equivariance of \equiperceptor{} with the exact equivariance of \equiactor{} to bound the full pipeline error.

\begin{assumption}[Lipschitz Continuity of \equiactor{}]
\label{ass:lipschitz}
\equiactor{} $\mathcal{A}$ is $L$-Lipschitz in its visual context input: for any two context representations $\mathbf{z}_1, \mathbf{z}_2$ with shared proprioceptive state and noise,
$\norm{\mathcal{A}(\mathbf{z}_1) - \mathcal{A}(\mathbf{z}_2)} \leq L \cdot \norm{\mathbf{z}_1 - \mathbf{z}_2}$.
This is a standard regularity assumption for neural networks with bounded weights; for transformer architectures, $L$ scales with the product of attention and MLP layer norms, and can be controlled via spectral normalization or weight decay~\citep{weiler2019general}.
\end{assumption}

\begin{corollary}[End-to-End Equivariance]
\label{cor:e2e}
Under~\Cref{ass:lipschitz}, the full \equivla{} policy $\pi(\mathbf{o}) = \mathcal{A}\bigl(\mathcal{P}(\mathbf{o})\bigr)$ satisfies:
\begin{equation}
    \norm{\pi(g \cdot \mathbf{o}) - \rho_a(g) \cdot \pi(\mathbf{o})}
    \;\leq\;
    L \cdot \Delta \cdot B(\mathbf{o}),
    \label{eq:e2e_bound}
\end{equation}
where $\rho_a$ denotes the $\SO(2)$ action on the action space, $\Delta$ is the representation defect of $\tilde{\tau}$, $B(\mathbf{o}) = \max_{h \in G} \norm{f_\theta(h \cdot \mathbf{x})}$ is the maximum ViT feature norm (with $\mathbf{x}$ the image component of~$\mathbf{o}$), and $L$ is the Lipschitz constant of \equiactor{}.
\end{corollary}

\begin{proof}
Write $\mathcal{P}(g \cdot \mathbf{o}) = \tilde{A}(g) \cdot \mathcal{P}(\mathbf{o}) + \boldsymbol{\varepsilon}$
with $\norm{\boldsymbol{\varepsilon}} \leq \Delta \cdot B$ by \Cref{thm:approx_equivariance}.
Since \equiactor{} is exactly equivariant~(\Cref{prop:flow_equivariance}) and the Equivariant Adapter preserves equivariance~(\Cref{prop:adapter_equivariance}):
\begin{align}
    \pi(g \cdot \mathbf{o})
    &= \mathcal{A}\bigl(\tilde{A}(g) \cdot \mathcal{P}(\mathbf{o}) + \boldsymbol{\varepsilon}\bigr) \\
    &= \underbrace{\mathcal{A}\bigl(\tilde{A}(g) \cdot \mathcal{P}(\mathbf{o})\bigr)}_{= \;\rho_a(g) \cdot \pi(\mathbf{o})}
    + \bigl[\mathcal{A}\bigl(\tilde{A}(g) \cdot \mathcal{P}(\mathbf{o}) + \boldsymbol{\varepsilon}\bigr) - \mathcal{A}\bigl(\tilde{A}(g) \cdot \mathcal{P}(\mathbf{o})\bigr)\bigr].
\end{align}
By the Lipschitz property, the bracketed remainder has norm $\leq L\norm{\boldsymbol{\varepsilon}} \leq L \cdot \Delta \cdot B$.
\end{proof}

\begin{remark}[Structure, Practical Implications, and Relation to Prior Work]
\label{rmk:unified}
The bound~\eqref{eq:e2e_bound} cleanly separates three independent factors controlling equivariance quality:
\begin{enumerate}[leftmargin=*,itemsep=1pt]
    \item $\Delta$: the spatial homomorphism defect, a pure geometric quantity determined by $G$ and the patch grid.
    For $C_4$ on a square grid, $\Delta = 0$ (exact).
    For $C_8$, only the $45^\circ$-offset coset contributes, and $\Delta$ can be computed in closed form for any given grid~(\Cref{tab:defect_values}).
    The per-element bound (left inequality in~\eqref{eq:approx_bound}) is significantly tighter when many $(g, h)$ pairs have zero defect, as the coset structure guarantees for the $C_4$ subgroup.
    \item $B$: the ViT feature magnitude, which is bounded for normalized inputs and architectures with LayerNorm.
    \item $L$: the Lipschitz constant of \equiactor{}, controlled by weight norms and architecture depth.
    For transformers with spectral normalization, $L$ scales polynomially rather than exponentially in depth.
\end{enumerate}
\equiactor{} alone (with invariant VLM context, i.e., without \equiperceptor{}) achieves exact equivariance by construction via steerable layers, since no approximate spatial reassignment is involved.
The full \equivla{} pipeline incurs approximation error from \equiperceptor{}'s nearest-neighbor reassignment, but benefits from richer, spatially equivariant visual features.
We validate empirically in~\Cref{sec:ablations} that the measured equivariance error is small and consistent with the bound structure.

\emph{Relation to prior work.}
The representation defect framework extends the exact equivariance guarantees of Frame Averaging~\citep{puny2021frame} to the approximate regime arising from spatial discretization, a setting not addressed in the original FA theory.
\citet{wang2025practical} apply FA to globally pooled ResNet features without spatial structure, making the exact FA guarantee trivially applicable; our token-level extension requires the approximate analysis because spatial permutations on finite grids break the homomorphism property.
In the broader landscape of approximate equivariance, \citet{park2024approximate} bound Q-function errors for approximately equivariant MDPs; our bound is analogous but operates at the representation level rather than the value-function level.
\end{remark}

\begin{corollary}[Sample Complexity Benefit of Equivariance]
\label{cor:sample_complexity}
Let $\mathcal{F}$ be a function class for the policy, and $\mathcal{F}_G = \set{f \in \mathcal{F} : f(g \cdot \mathbf{o}) = \rho_a(g) f(\mathbf{o})\;\forall g \in G}$ its $G$-equivariant subclass.
Then the empirical Rademacher complexity satisfies~\citep{elesedy2021provably}:
\begin{equation}
    \hat{\mathcal{R}}_n(\mathcal{F}_G) \;\leq\; \frac{1}{|G|}\,\hat{\mathcal{R}}_n(\mathcal{F}).
    \label{eq:rademacher_bound}
\end{equation}
By standard generalization bounds, this implies the $G$-equivariant class requires a factor of $|G|$ fewer samples to achieve the same generalization error, yielding up to an $8\times$ reduction for $C_8$.
For the approximately equivariant \equivla{} pipeline, the effective reduction is modulated by the equivariance residual: if $\norm{\boldsymbol{\varepsilon}} \leq \epsilon_0$, the generalization gap is bounded by $\frac{1}{|G|}\hat{\mathcal{R}}_n(\mathcal{F}) + \mathcal{O}(\epsilon_0)$, interpolating between the full equivariant benefit ($\epsilon_0 = 0$) and the unconstrained baseline ($\epsilon_0 \to \infty$).
\end{corollary}

\subsection{Empirical Equivariance Error}
\label{app:empirical_equivariance}
We complement the theoretical bounds of \Cref{thm:approx_equivariance} 
with empirical measurements of the equivariance error $\epsilon_{\mathrm{eq}}$ (defined in~\Cref{eq:epsilon-eq}).
We evaluated over $M{=}500$ randomly sampled observations from the LIBERO validation set across all $g \in C_8$. To construct the rotated observation
$g \cdot \mathbf{o}$, we apply rotation $g$ to the top-down camera
image and end-effector pose, while keeping the wrist camera image
fixed, as the wrist view rotates with the end-effector and is thus
invariant to base-frame scene rotation.

\begin{table}[h]
  \centering
  \caption{%
    Empirical equivariance error $\epsilon_{\mathrm{eq}}$
    (mean $\pm$ std) over $M{=}500$ observations and all
    $g \in C_8$, measured on the LIBERO validation set.
  }
  \label{tab:equivariance-error}
  \vspace{4pt}
  \small
  \setlength{\tabcolsep}{5pt}
  \renewcommand{\arraystretch}{1.1}
  \begin{tabular}{lc}
    \toprule
    Method & $\epsilon_{\mathrm{eq}}$ (mean $\pm$ std) $\downarrow$ \\
    \midrule
    GR00T N1.5~\citep{bjorck2025gr00t}  & $7.754 \pm 3.572$ \\
    GR00T N1.5 + \equiactor{}                & $0.837 \pm 0.738$ \\
    \rowcolor{blue!8} \textbf{\equivla{}} (ours)             & $\mathbf{0.284 \pm 0.134}$ \\
    \bottomrule
  \end{tabular}
\end{table}

GR00T~N1.5 exhibits large equivariance error ($7.754 \pm 3.572$), consistent
with the absence of any equivariance constraint across vision, state, and
action components. GR00T~N1.5 + \equiactor{} reduces the error by $9.3\times$ to
$0.837 \pm 0.738$: the equivariant action head correctly transforms actions
under $C_8$ rotations, and the proprioceptive state is handled equivariantly
through the steerable layers; however, the non-equivariant visual backbone
still produces rotation-inconsistent visual features, which propagate into the
action head as a residual violation. The full \equivla{} further reduces error to $\mathbf{0.284 \pm 0.134}$
(a $27.3\times$ reduction over the baseline) as \equiperceptor{}
suppresses the upstream vision equivariance violation via token-level
Frame Averaging, leaving only the small approximation error from the
frozen ViT's spatial discretization, consistent with \Cref{thm:approx_equivariance}.

\Cref{fig:polar-equivariance} reports the per-group equivariance
error $\epsilon_{\mathrm{eq}}$ at each $C_8$ rotation angle,
complementing the aggregate statistics in
\Cref{tab:equivariance-error}. The per-angle breakdown confirms
three predictions from the theoretical analysis:
\begin{enumerate}[leftmargin=*,itemsep=1pt]
    \item \textbf{GR00T N1.5} (no equivariance constraint) shows
    error growing monotonically from $0^\circ$ to $180^\circ$, consistent
    with the cosine distance of rotations on a non-equivariant
    feature space.
    \item \textbf{GR00T N1.5 + \equiactor{}} shows an oscillating
    pattern with higher error at $45^\circ$-offset angles
    ($45^\circ, 135^\circ, 225^\circ, 315^\circ$) than at $C_4$ angles
    ($90^\circ, 180^\circ, 270^\circ$), consistent with
    \Cref{prop:exact_cases}(b--c): the $C_4$ subgroup has zero
    defect, while the $45^\circ$-offset coset contributes all
    non-zero defect terms.
    \item \textbf{\equivla{}} remains near-zero and approximately
    flat across all angles, confirming that \equiperceptor{} suppresses
    the upstream vision equivariance violation uniformly across
    the group.
\end{enumerate}

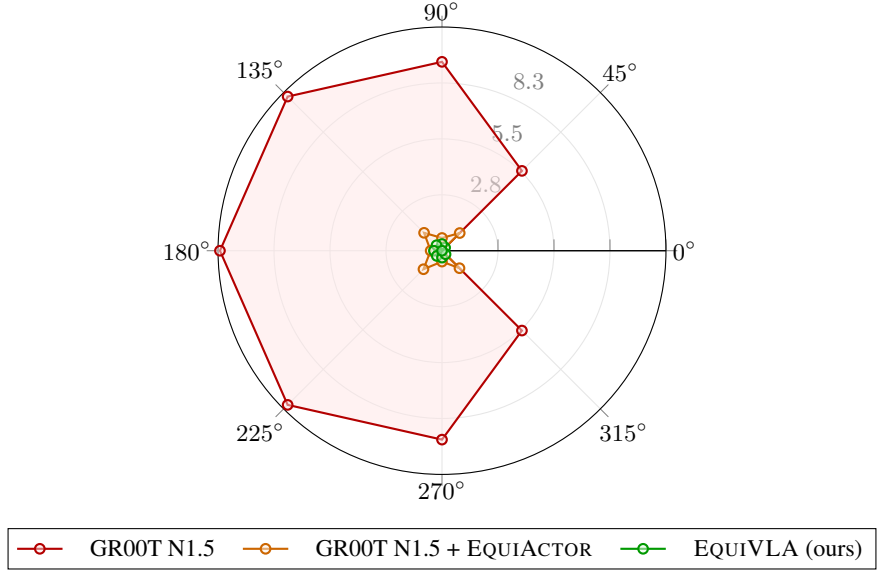
\begin{figure}[!ht]
  \centering
  \scriptsize
  \begin{tikzpicture}
  \begin{polaraxis}[
    width=7.5cm, height=7.5cm,
    xtick={0,45,90,135,180,225,270,315},
    xticklabels={
      $0^\circ$,$45^\circ$,$90^\circ$,$135^\circ$,
      $180^\circ$,$225^\circ$,$270^\circ$,$315^\circ$
    },
    xticklabel style={font=\small},
    ymin=0, ymax=11,
    ytick={2.75,5.5,8.25,11},
    yticklabels={,,,},
    grid=both,
    grid style={gray!20, line width=0.4pt},
    legend style={
      at={(0.5,-0.12)},
      anchor=north,
      legend columns=3,
      font=\small,
      column sep=8pt,
    },
  ]
  \node[font=\small, gray, anchor=south west]
    at (axis cs:67.5, 2.75) {$2.8$};
  \node[font=\small, gray, anchor=south west]
    at (axis cs:67.5, 5.5)  {$5.5$};
  \node[font=\small, gray, anchor=south west]
    at (axis cs:67.5, 8.25) {$8.3$};
  \node[font=\small, gray, anchor=south west]
    at (axis cs:67.5, 11)   {$11$};
  \addplot[red!70!black, thick, fill=red!10, fill opacity=0.5,
         mark=*, mark size=1.8pt] coordinates {
      (0,0)(45,5.553)(90,9.289)(135,10.721)
      (180,10.906)(225,10.721)(270,9.289)(315,5.552)(360,0)
    };
    \addlegendentry{GR00T N1.5}
    \addplot[orange!80!black, thick, fill=orange!10, fill opacity=0.5,
             mark=*, mark size=1.8pt] coordinates {
      (0,0)(45,1.231)(90,0.627)(135,1.245)
      (180,0.557)(225,1.287)(270,0.534)(315,1.215)(360,0)
    };
    \addlegendentry{GR00T N1.5 + \equiactor{}}
    \addplot[green!60!black, thick, fill=green!10, fill opacity=0.5,
             mark=*, mark size=1.8pt] coordinates {
      (0,0)(45,0.216)(90,0.328)(135,0.380)
      (180,0.397)(225,0.362)(270,0.339)(315,0.247)(360,0)
    };
  \addlegendentry{\equivla{} (ours)}
  \end{polaraxis}
  \end{tikzpicture}
  \caption{%
    Per-group equivariance error $\epsilon_{\mathrm{eq}}$ at each
    $C_8$ rotation angle, measured over $M{=}500$ observations
    from the LIBERO validation set. GR00T~N1.5 (red) grows
    monotonically toward $180^\circ$. GR00T~N1.5 + \equiactor{} (orange)
    shows an oscillating pattern with higher error at
    $45^\circ$-offset angles, consistent with the coset structure
    in \Cref{prop:exact_cases}(b). \equivla{} (green) remains
    near-zero across all angles.
  }
  \label{fig:polar-equivariance}
\end{figure}

\section{Equivariant Adapter: Invariant Stream Update}
\label{app:adapter_inv_stream}

The invariant tokens in $\mathbf{z}^{\mathrm{ctx}}$ (language and wrist-camera streams) are updated via a second per-token gate that incorporates information from the equivariant stream.
An equivariant summary $\mathbf{s}^{\mathrm{eq}} = \mathrm{mean}_{N,G}(\mathbf{z}^{\mathrm{out}}_{\mathrm{eq}})$ is computed by pooling over both token positions and group channels, yielding an invariant vector.
For each invariant context token $\mathbf{z}^{\mathrm{ctx}}_i$, the update is:
\begin{equation}
    \mathbf{z}^{\mathrm{out}}_i
    = \boldsymbol{\beta}_i \odot \mathbf{W}_{\mathrm{tok}}(\mathbf{z}^{\mathrm{ctx}}_i)
    + (\mathbf{1} - \boldsymbol{\beta}_i)
      \odot \mathbf{W}_{\mathrm{ctx}}(\mathbf{s}^{\mathrm{eq}}),
    \label{eq:adapter_inv_out}
\end{equation}
where $\boldsymbol{\beta}_i = \sigma\bigl(\mathbf{W}_{\mathrm{gate\_inv}}[\mathbf{z}^{\mathrm{ctx}}_i;\, \mathbf{s}^{\mathrm{eq}}]\bigr)$ is an invariant gate.
Since both $\mathbf{z}^{\mathrm{ctx}}_i$ and $\mathbf{s}^{\mathrm{eq}}$ are $G$-invariant, the gate $\boldsymbol{\beta}_i$ is invariant and $\mathbf{z}^{\mathrm{out}}_i$ is invariant by construction.
The final adapter output concatenates all streams:
\begin{equation}
    \mathbf{z}^{\mathrm{out}} = \bigl[\mathbf{z}^{\mathrm{out}}_{\mathrm{eq}};\; \mathbf{z}^{\mathrm{out}}_{\mathrm{inv}};\; \mathbf{z}^{\mathrm{out}}_{\mathrm{lang}}\bigr],
    \label{eq:adapter_full_out}
\end{equation}
where $\mathbf{z}^{\mathrm{out}}_{\mathrm{eq}}$ is $G$-equivariant (\Cref{eq:adapter_out} in the main text) and the remaining tokens are $G$-invariant.

\section{Controlled Orientation-Shift Experiment}
\label{app:orientation_shift}
While \Cref{sec:ablations} validates equivariance error on the
discrete $C_8$ grid, a key open question is whether \equivla{}
generalizes to \emph{out-of-subgroup} angles, rotations that
lie strictly between $C_8$ grid points and were never explicitly
seen during training augmentation. Since \equiperceptor{} achieves
only approximate $\SO(2)$ equivariance and training augmentation
covers only 8 discrete angles, some performance degradation at
larger rotations is expected for all models. The relevant question
is whether equivariant models degrade \emph{more gracefully} than
the non-equivariant baseline, that is, whether geometric
inductive bias provides a meaningful advantage even at unseen
rotation angles. We design a controlled orientation-shift
experiment to directly test this hypothesis.

\textbf{Setup.}
We evaluate on LIBERO-Object at 11 rotation angles
$\{-25^\circ, -20^\circ, \ldots, +20^\circ, +25^\circ\}$ in
$5^\circ$ increments, reporting success rate averaged over 50
rollouts per angle across all 10 tasks. We restrict the range
to $\pm25^\circ$ for two reasons: first, LIBERO's camera is
mounted at a slight downward angle rather than directly overhead,
so scene rotation induces increasing perspective distortion at
larger angles; second, the training demonstrations exhibit limited rotational diversity, meaning large rotations push observations outside the training distribution regardless of architectural equivariance. Within $\pm25^\circ$, the evaluation isolates the effect of geometric inductive bias rather than out-of-distribution generalization more broadly.
At each angle $\theta$, we apply a planar rotation of $\theta$ 
degrees about the scene center to the object positions and 
end-effector pose in the proprioceptive state, while keeping 
all camera images unchanged. This simulates a scene in which 
objects and the robot's end-effector appear at rotated 
configurations relative to the canonical training distribution.

\begin{figure}[h]
\centering
\setlength{\tabcolsep}{2pt}
\renewcommand{\arraystretch}{0.9}
\begin{tabular}{ccccc}
    \includegraphics[width=0.18\linewidth]{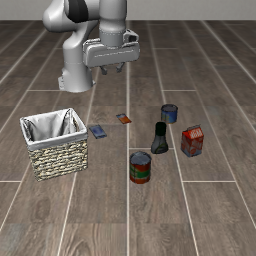} &
    \includegraphics[width=0.18\linewidth]{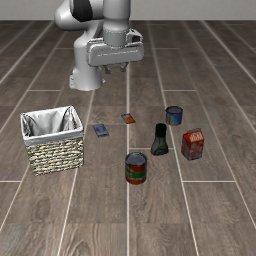} &
    \includegraphics[width=0.18\linewidth]{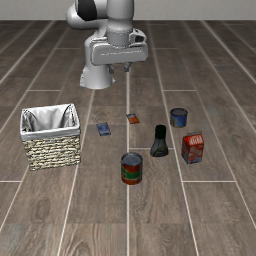} &
    \includegraphics[width=0.18\linewidth]{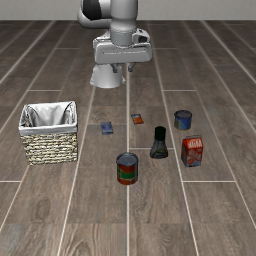} &
    \includegraphics[width=0.18\linewidth]{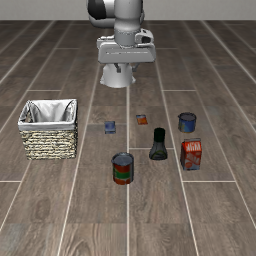} \\
    \small $-25^\circ$ & \small $-20^\circ$ & \small $-15^\circ$ & 
    \small $-10^\circ$ & \small $-5^\circ$ \\[4pt]
    \includegraphics[width=0.18\linewidth]{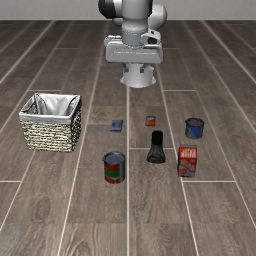} &
    \includegraphics[width=0.18\linewidth]{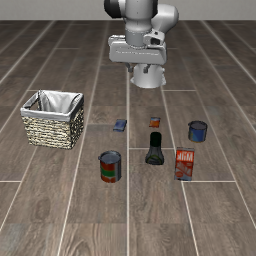} &
    \includegraphics[width=0.18\linewidth]{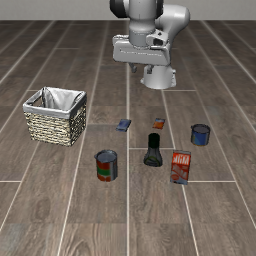} &
    \includegraphics[width=0.18\linewidth]{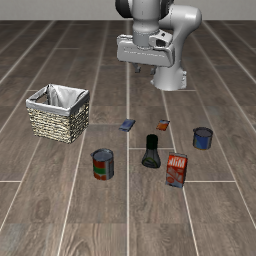} &
    \includegraphics[width=0.18\linewidth]{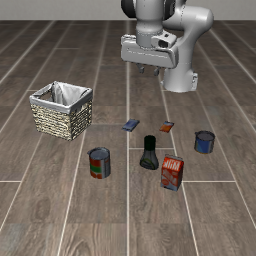} \\
    \small $+5^\circ$ & \small $+10^\circ$ & \small $+15^\circ$ & 
    \small $+20^\circ$ & \small $+25^\circ$ \\
\end{tabular}
\caption{Example LIBERO-Object scenes at each rotation angle used in the 
controlled orientation-shift experiment. Row 1: negative rotations 
($-25^\circ$ to $-5^\circ$); Row 2: positive rotations ($+5^\circ$ to 
$+25^\circ$). Object positions and end-effector pose are rotated while 
camera views remain unchanged.}
\label{fig:orientation_examples}
\end{figure}

\begin{figure}[h]
    \centering
    \includegraphics[width=0.8\linewidth]{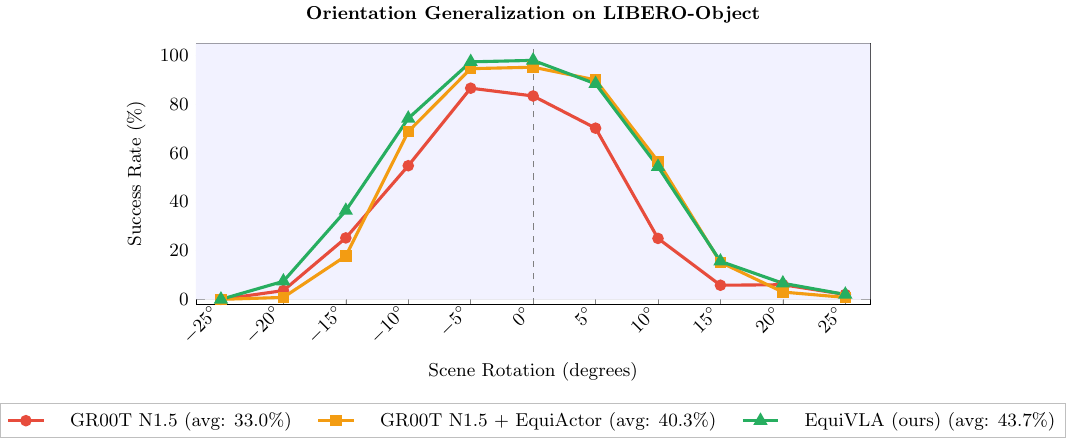}
    \caption{Success rate on LIBERO-Object under controlled 
    scene rotations from $-25^\circ$ to $+25^\circ$. 
    \equivla{} and GR00T~N1.5 + \equiactor{} maintain higher 
    and flatter performance across the rotation range compared 
    to GR00T~N1.5, which peaks near $0^\circ$ and degrades 
    rapidly at larger rotations.}
    \label{fig:orientation_shift}
\end{figure}

\textbf{Results.}
\Cref{fig:orientation_shift} reveals a clear separation between
equivariant and non-equivariant models within the moderate
rotation range. GR00T~N1.5 peaks near the canonical orientation
($-5^\circ$: $86.6\%$, $0^\circ$: $83.4\%$) and degrades
rapidly, dropping to $25.0\%$ at $+10^\circ$ and near zero
at $\pm25^\circ$, confirming the absence of rotational
generalization. Both equivariant variants maintain substantially
higher performance within $[-15^\circ, +15^\circ]$:
GR00T~N1.5 + \equiactor{} achieves $95.2\%$ at $0^\circ$ and
retains $56.6\%$ at $+10^\circ$; \equivla{} reaches $98.0\%$
at $0^\circ$ and $54.4\%$ at $+10^\circ$. Beyond $\pm15^\circ$,
all models degrade substantially as large rotations push
observations outside the training distribution, consistent
with the camera and dataset limitations described in the
setup. Averaged across all angles, \equivla{} achieves
$43.7\%$ vs.\ $40.3\%$ and $33.0\%$, confirming that
equivariant structure provides a consistent advantage within
the practical rotation range. The progressive improvement
mirrors the equivariance error reduction in
\Cref{app:empirical_equivariance}, providing direct empirical
evidence that architectural equivariance translates to
rotation generalization in practice.

\end{document}